\pdfoutput=1

\documentclass[11pt]{article}

\usepackage{ACL2023}

\usepackage{times}
\usepackage{latexsym}

\usepackage[T1]{fontenc}

\usepackage[utf8]{inputenc}

\usepackage{microtype}

\usepackage{inconsolata}

\usepackage[ruled,linesnumbered]{algorithm2e}
\usepackage{enumitem}
\usepackage{booktabs}
\usepackage{xcolor}
\usepackage{caption}
\usepackage{subcaption}
\usepackage{amsmath}
\usepackage{multirow}
\usepackage{amsfonts}

\usepackage{graphics}
\usepackage{graphicx}
\usepackage{amssymb}

\usepackage[normalem]{ulem}
\useunder{\uline}{\ul}{}

\usepackage{amsthm}

\newtheorem{theorem}{Theorem}

\usepackage{color}
\usepackage{colortbl}
\definecolor{mygray}{gray}{.9}
\definecolor{LightCyan}{rgb}{0.88,1,1}

\newcommand{\nop}[1]{}

\newcommand{\model}{{nPUGraph}\xspace}

\title{Noisy Positive-Unlabeled Learning with Self-Training for \\ Speculative Knowledge Graph Reasoning}

\author{Ruijie Wang, Baoyu Li, Yichen Lu, Dachun Sun, Jinning Li, Yuchen Yan, \\ {\bf Shengzhong Liu}, {\bf Hanghang Tong}, {\bf Tarek F. Abdelzaher} \\
  University of Illinois Urbana-Champaign, IL, USA \\
   \texttt{ \small \{ruijiew2,baoyul2,yichen14,dsun18,jinning4,yucheny5,sl29,htong,zaher\}@illinois.edu}
}
\begin{document}
\maketitle

\begin{abstract}

This paper studies speculative reasoning task on real-world knowledge graphs (KG) that contain both \textit{false negative issue} (i.e., potential true facts being excluded) and \textit{false positive issue} (i.e., unreliable or outdated facts being included). State-of-the-art methods fall short in the speculative reasoning ability, as they assume the correctness of a fact is solely determined by its presence in KG, making them vulnerable to false negative/positive issues. The new reasoning task is formulated as a noisy Positive-Unlabeled learning problem. We propose a variational framework, namely \model, that jointly estimates the correctness of both collected and uncollected facts  (which we call \textit{label posterior}) and updates model parameters during training. The label posterior estimation facilitates speculative reasoning from two perspectives. First, it improves the robustness of a label posterior-aware graph encoder against false positive links. Second, it identifies missing facts to provide high-quality grounds of reasoning. They are unified in a simple yet effective self-training procedure. Empirically, extensive experiments on three benchmark KG and one Twitter dataset with various degrees of false negative/positive cases demonstrate the effectiveness of \model. 

\end{abstract}
\maketitle
\section{Introduction}
\label{sec:intro}
Knowledge graphs (KG), which store real-world facts in triples (\textit{head entity}, \textit{relation}, \textit{tail entity}), have facilitated a wide spectrum of knowledge-intensive applications~\cite{wang2018ace,KBQA,KG4Social,RippleNet,wang2022rete}. Automatically reasoning facts based on observed ones, a.k.a. Knowledge Graph Reasoning (KGR)~\cite{TransE}, becomes increasingly vital since it allows for expansion of the existing KG at a low cost.


\begin{figure}
    \centering
    \includegraphics[width = 1.\linewidth]{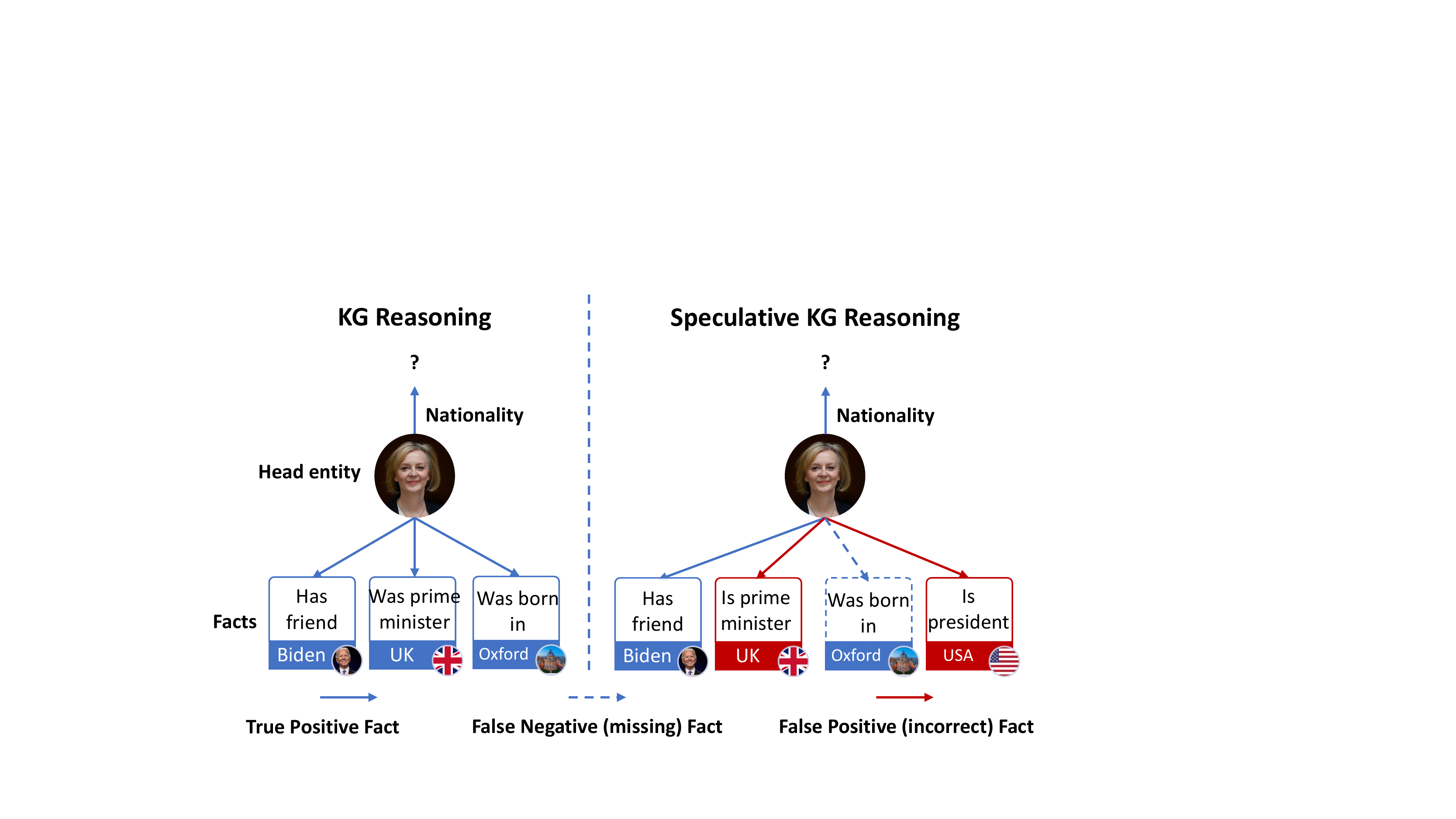}
    \caption{An illustrative example of speculative KG reasoning. Blue solid lines denote the true positive fact, blue dashed lines denote the \textit{false negative} (missing) fact, and red solid lines denote the \textit{false positive} (incorrect) fact. We aim to mitigate the false negative/positive issues to enable speculative KG reasoning.}
    \label{fig:task}
\end{figure}


Numerous efforts have been devoted to KGR task~\cite{TransE,TransR,ComplEX,RotatE,RE-GCN}, which assume the correctness of a fact is solely determined by its presence in KG. They ideally view facts included in KG as positive samples and excluded facts as negative samples. However, most real-world reasoning has to be performed based on sparse and unreliable observations, where there may be true facts excluded or false facts included. Reasoning facts based on sparse and unreliable observations (which we call {\em speculative KG reasoning\/}) are still underexplored.

In this paper, we aim to enable the speculative reasoning ability on real-world KG. The fulfillment of the goal needs to address two commonly existing issues, as shown in Figure~\ref{fig:task}: 1) \textbf{The false negative issue} (i.e., sparse observation): Due to the graph incompleteness, facts excluded from the KG can be used as implicit grounds of reasoning. This is particularly applicable to non-obvious facts. For example, personal information such as the birthplace of politicians may be missing when constructing a political KG, as they are not explicitly stated in the political corpus~\cite{PUDA}. However, it can be critical while reasoning personal facts like nationality. 2) \textbf{The false positive issue} (i.e., noisy observation): Facts included in the KG may be unreliable and should not be directly grounded without inspection. It can happen when relations between entities are incorrectly collected or when facts are extracted from outdated or unreliable sources. For example, Mary Elizabeth is no longer  the Prime Minister of the United Kingdom, which may affect the reasoning accuracy of her current workplace. These issues generally affect both one-hop reasoning~\cite{TransE} and multi-hop reasoning~\cite{KBQA}. The main focus of this paper is investigating the one-hop speculative reasoning task as it lays the basis for complicated multi-hop reasoning capability. 



Speculative KG reasoning differs from conventional KG reasoning in that the correctness of each collected/uncollected fact needs to be dynamically estimated as part of the learning process, such that the grounds of reasoning can be accordingly calibrated. 
Unfortunately, most existing work, if not all, lacks such inspection capability. Knowledge graph embedding methods~\cite{TransE,TransR,DistMult,ComplEX,RotatE} and graph neural network (GNN) methods~\cite{R-GCN,ConvE,ConvKB,CompGCN,RE-GCN} can easily overfit the false negative/positive cases because of their training objective that ranks the collected facts higher than other uncollected facts in terms of plausibility. Recent attempts on uncertain KG~\cite{UKGE,GTransE} measure the uncertainty scores for facts, which can be utilized to detect false negative/positive samples. However, they explicitly require the ground truth uncertainty scores as supervision for reasoning model training, which are usually unavailable in practice. 
Motivated by these observations, we formulate the speculative KG reasoning task as a noisy Positive-Unlabeled learning problem. The facts contained in the KG are seen as noisy positive samples with a certain level of label noise, and the facts excluded from the KG are treated as unlabeled samples, which include both negative ones and possible factual ones. 
Instead of determining the correctness of facts before training the reasoning model without inspection, we learn the two perspectives in an end-to-end training process.
To this end, we propose \model, a novel variational framework that regards the underlying correctness of collected/uncollected facts in the KG as latent variables for the reasoning process. We jointly update model parameters and estimate the posterior likelihood of the correctness of each collected/uncollected fact ({referred to as \em label posterior\/}), through maximizing a theoretical lower bound of the log-likelihood of each fact being collected or uncollected. 

The estimated label posterior further facilitates the speculative KG reasoning from two aspects: 1) It removes false positive facts contained in KG and improves the representation quality. We accordingly propose a label posterior-aware encoder to incorporate information only from entity neighbors induced by facts with a high posterior probability, under the assumption that the true positive facts from the collected facts provide more reliable information for reasoning. 2) It complements the grounds of reasoning by selecting missing but possibly plausible facts with high label posterior, which are iteratively added to acquire more informative samples for model training. These two procedures are ultimately unified in a simple yet effective self-training strategy that alternates between the  \textit{data sampling based on latest label posteriors} and the \textit{model training based on latest data samples}. Empirically, \model outperforms eleven state-of-the-art baselines on three benchmark KG data and one Twitter data we collected by large margins. Additionally, its robustness is demonstrated in speculative reasoning on data with multiple ratios of false negative/positive cases. 



Our contributions are summarized as follows: (1) We open up a practical but underexplored problem space of speculative KG reasoning, and formulate it as a noisy Positive-Unlabeled learning task; (2) We take the first step in tackling this problem by proposing a variational framework \model to jointly optimize reasoning model parameters and estimate fact label posteriors; (3) We propose a simple yet effective self-training strategy for \model to simultaneously deal with false negative/positive issues; (4) We perform extensive evaluations to verify the effectiveness of \model on both benchmark KG and Twitter interaction data with a wide range of data perturbations. 
 \section{Preliminaries}

\subsection{Speculative Knowledge Graph Reasoning} 

A knowledge graph (KG) is denoted as $\mathcal{G} = \{(e_h, r, e_t) \} \subseteq \mathcal{S}$, where $\mathcal{S} = \mathcal{E} \times \mathcal{R} \times \mathcal{E}$ denotes triple space, $\mathcal{E}$ denotes the entity set, $\mathcal{R}$ denotes the relation set. Each triple $s = (e_h, r, e_t)$ refers to that a head entity $e_h \in \mathcal{E}$ has a relation $r \in \mathcal{R}$ with a tail entity $e_t \in \mathcal{E}$. Typically, a score function $\psi(s;\Theta)$, parameterized by $\Theta$, is designed to measure the plausibility of each potential triple $s = (e_h, r, e_t)$, and to rank the most plausible missing ones to complete KG during inference~\cite{TransE,RotatE}. The goal of speculative KG reasoning is to infer the most plausible triple for each incomplete triple $(e_h, r, e_?)~\text{or}~(e_?, r, e_t)$ given by sparse and unreliable observations in $\mathcal{G}$. In addition, it requires correctness estimation for each potential fact collected or uncollected by $\mathcal{G}$. 



\subsection{Noisy Positive-Unlabeled Learning}
Positive-Unlabeled (PU) learning is a learning paradigm for training a model when only positive and unlabeled data is available~\cite{PULearning}. We formulate the speculative KG reasoning task as a noisy Positive-Unlabeled learning problem, where the positive set contains potentially label noise from false facts~\cite{nPU}.

\noindent \textbf{PU Triple Distribution}. For the speculative KG reasoning task, we aim to learn a binary classifier that maps a triple space $\mathcal{S}$ to a label space $\mathcal{Y}=\{0,1\}$. Data are split as labeled (collected)\footnote{In this paper, we interchangeably use the term \textit{labeled/unlabeled} and \textit{collected/uncollected} with no distinction.} triples $s^l\in\mathcal{S}^L$ and unlabeled (uncollected) triples $s^u\in\mathcal{S}^U$. The labeled triples are considered noisy positive samples with a certain level of label noise. The distribution of labeled triples can be represented as follows:

\begin{equation}
    \small
    s^l \sim \beta \phi_1^l(s^l) + (1-\beta) \phi_0^l(s^l),
\end{equation}
\noindent
where $\phi_y^l$ denotes the probability of being collected over triple space $\mathcal{S}$ for the positive class ($y=1$) and negative class ($y=0$), and $\beta\in [0,1)$ denotes the proportion of true positive samples in labeled data. Unlabeled triples include both negative samples and possible factual samples. The distribution of unlabeled samples can be represented as follows:
\begin{equation}
    \small
    s^u \sim \alpha \phi_1^u(s^u) + (1-\alpha) \phi_0^u(s^u),
\end{equation}
\noindent
where $\phi_y^u = 1- \phi_y^l$ denotes the probability of being uncollected, $\alpha\in [0,1)$ denotes the positive class prior, i.e., the proportion of positive samples in unlabeled data.

\noindent \textbf{PU Triple Construction}. We then discuss the construction of $\mathcal{S}^L$ and $\mathcal{S}^U$ based on the collected KG $\mathcal{G}$. Triples in $\mathcal{G}$ naturally serve as labeled samples with a ratio of noise, i.e., $\mathcal{S}^L = \mathcal{G}$. For unlabeled set $\mathcal{S}^U$, However, directly using $\mathcal{S}\setminus \mathcal{G}$ as the unlabeled set $\mathcal{S}^U$ would result in too many unlabeled samples for training due to the large number of possible triples in triple space $\mathcal{S}$. Following~\cite{PUDA}, we construct $\mathcal{S}^U$ as follows: For each labeled triple $s_i^l = (e_h, r, e_t)$, we construct $K$ unlabeled triples $s_{ik}^u$ by replacing the head and tail respectively with other entities: $s_{ik}^u = (e_h, r, e_k^{-})~\text{or}~(e_k^{-}, r, e_t)$, where $e_k^{-}$ is the selected entity that ensures $s_{ik}^u \notin \mathcal{S}^L$. Initially, the construction can be randomized. During the training process, it is further improved by selecting unlabeled samples with high label posterior in a self-training scheme, so as to cover positive samples in the unlabeled set to the greatest extent.

\section{Methodology}

\begin{figure}[t]
    \centering
    \includegraphics[width = 1.\linewidth]{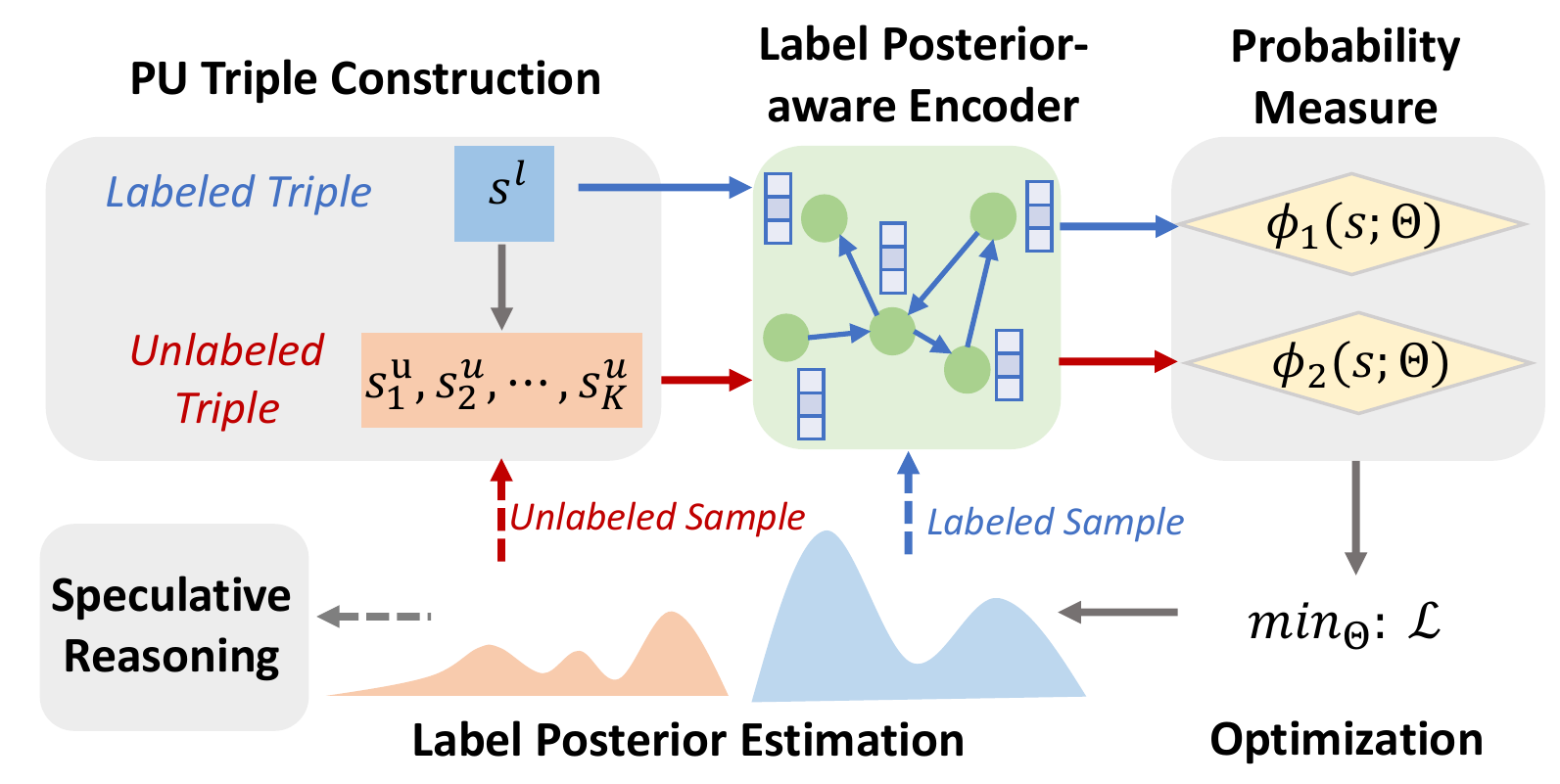}
    \caption{\model overview. It jointly optimizes parameters and estimates label posterior, to detect false negative/positive cases for the encoder and self-training.}
    \label{fig:framework}
\end{figure}

\subsection{Overview}
 Our approach views underlying triple labels (positive/negative) as latent variables, influencing the  collection probability. Unlike the common objective of reasoning training that ranks the plausibility of the collected triples higher than uncollected ones, we instead maximize the data likelihood of each potential triple being collected or not. To this end, as shown in Figure~\ref{fig:framework}, we propose \model framework to jointly optimize parameters and infer the label posterior. During the training process, the latest label posterior estimation can be utilized by a label posterior-aware encoder, which improves the quality of representation learning by only integrating information from the entity neighbors induced by true facts. Finally, a simple yet effective self-training strategy based on label posterior is proposed, which can dynamically update neighbor sets for the encoder and sample unlabeled triplets to cover positive samples in the unlabeled set to the greatest extent for model training.

 The remaining of this section is structured as follows: Section~\ref{sec:noisyPU} first formalizes the learning objective and the variational framework for likelihood maximization. Section~\ref{sec:encoder} details the label posterior-aware encoder for representation learning, followed by Section~\ref{sec:st} that introduces the self-training strategy.

\subsection{Noisy PU Learning on KG}
\label{sec:noisyPU}

Due to the false negative/positive issues, the correctness of a fact ($y$) is not solely determined by its presence in a knowledge graph. \model addresses the issue by treating the underlying label as a latent variable that influences the probability of being collected or not. We, therefore, set maximizing the data collection likelihood as our objective. In such a learning paradigm, the assumptions that collected triples are correct $p(y=1|s^l) = 1$ and uncollected triples are incorrect $p(y=0|s^u)=1$ are removed. We aim to train a model on labeled triples $\mathcal{S}^L$ and unlabeled ones $\mathcal{S}^U$, and infer the label posterior $p(y|s^u)$ and $p(y|s^l)$ at the same time by data likelihood maximization. The latest label posterior can help to detect false negative/positive cases during model training.


We first derive our training objective. To be more formal, the log-likelihood of each potential fact being collected or not is lower bounded by Eq.~\eqref{eq:llc}, which is given by Theorem~\ref{the:llc}. 


\begin{theorem}
\label{the:llc}

The log-likelihood of the complete data $\log p(\mathbf{S})$ is lower bounded as follows:
\begin{equation}
\scriptsize
    \begin{split}
        \cr &\log p(\mathbf{S})\geq \underset{q(\mathbf{Y})}{\mathbb{E}} \left[\log p(\mathbf{S}|\mathbf{Y})\right] - \mathbb{KL}(q(\mathbf{Y})\| p(\mathbf{Y})) \
        \cr& = \underset{s^l\in\mathcal{S}^L}{\mathbb{E}}\left[w^{l}\log[\phi_1^l(s^l)]+(1-w^{l})\log[\phi_0^l(s^l)]\right] \
        \cr& + \underset{s^u\in\mathcal{S}^U}{\mathbb{E}}\left[(w^{u}\log[\phi_1^u(s^u)]+(1-w^{u})\log[\phi_0^u(s^u)]\right] \
        \cr& - \mathbb{KL}(\mathbf{W}^U\|\mathbf{\Tilde{W}}^U) - \mathbb{KL}(\mathbf{W}^L\|\mathbf{\Tilde{W}}^L) - \frac{\|\mathbf{W}^L\|_1}{|\mathcal{S}^L|} - \frac{\|\mathbf{W}^U\|_1}{|\mathcal{S}^U|},
    \end{split}
    \label{eq:llc}
\end{equation}
\noindent
where $\mathbf{S}$ denotes all labeled/unlabeled triples, $\mathbf{Y}$ is the corresponding latent variable indicating the positive/negative labels for triples, $\mathbf{W}^U = \{w^{u}_i\}$ denotes the point-wise probability for the uncollected triples being positive, $\mathbf{W}^L = \{w^{l}_i\}$ denotes the probability for the collected triples being positive. $\mathbf{\Tilde{W}}^U$ and $\mathbf{\Tilde{W}}^U$ are the approximation of the collection probability for uncollected/collected triples respectively, produced by \model based on the latest parameters.
\end{theorem}

\begin{proof} Refer to Appendix~\ref{ap:proof} for proof. 
\end{proof}
We treat label $\mathbf{Y}$ as a latent variable and derive the lower bound for the log-likelihood, which is influenced by the prior knowledge of positive class prior $\alpha$ and true positive ratio $\beta$. Thus, maximizing the lower bound can jointly optimize model parameters and infer the posterior label distribution, $\mathbf{W}^U$ and $\mathbf{W}^L$. Such a learning process enables us to avoid false negative/positive issues during model training since it considers $\phi_0^l$ (one negative triple is collected) and $\phi_1^u$ (one positive triple is missing) as non-zero probability, which are determined by the latest label posterior during model training.

\noindent \textbf{Probability Measure}.
We then specify the probability measures for positive/negative triples being collected, i.e., $\phi_1^l(\cdot)$ and $\phi_0^l(\cdot)$ ($\phi_y^u(\cdot) = 1 - \phi_y^l(\cdot)$ for $y = 1/0$). To better connect to other methods utilizing score functions for KGR, we hereby utilize the sigmoid function $\sigma(\cdot)$ to directly transform the score function $\psi(s;\Theta)$, parameterized by model parameters $\Theta$, to probability:
\begin{equation}
    \small
    \phi_1^l(s) = \sigma(\psi_1(s;\Theta)), ~~~ \phi_0^l(s) = \sigma(\psi_0(s;\Theta)),
    \label{eq:prob}
\end{equation}
\noindent
we hereby utilize two score functions $\psi_1(s;\Theta)$ and $\psi_0(s;\Theta)$ to measure the positive/negative triples being collected, as the influencing factors based on triple information can be different. We utilize two neural networks to approximate the probability measure, which will be detailed in Section~\ref{sec:encoder}.

Since we aim to detect the potential existence of positive triples in an unlabeled set, it is unnecessary to push the collection probability of all uncollected triple $\phi_y^l(s^u)$ to 0 ($\phi_y^u(s^u)$ to 1). A loose constraint is that we force the uncollection probability of a collected triple $s^l$ lower than its corresponding uncollected triples $s^u$: $\phi_y^u(s^l) < \phi_y^u(s^u)$. Therefore, we adopt the pair-wise ranking measure $\phi^\star_y(s^l,s^u)$ to replace $\phi_y^u(s^u)$ as follows:
\begin{equation}
    \small
    \phi_y^u(s^u) \rightarrow \phi^\star_y(s^l,s^u) = \sigma(\psi_y(s^u;\Theta) - \psi_y(s^l;\Theta)).
    \label{eq:star_0}
\end{equation}
\noindent \textbf{Maximum Probability Training}. We then derive the training objective based on Eq.~\eqref{eq:llc} The first part of Eq.~\eqref{eq:llc} measures the probability of data being collected/uncollected. Concretely, given each collected triple $s_i^l\in\mathcal{S}^L$ and its corresponding $K$ uncollected triples $s^u_{ik}\in\mathcal{S}^U$, We denote the loss function measuring the probability as $\mathcal{L}_{triple}$:
\begin{equation}
\small
    \begin{split}
        \cr\mathcal{L}_{triple} & = -\frac{1}{K|\mathcal{S}^L|}\underset{i}{\sum}\underset{k}{\sum} (w^{l}_i\log [\phi_1^l(s_i^l)] \
        \cr& +(1-w_i^{l})\log [\phi_0^l(s_i^l)] +   w^{u}_{ik}\log[\phi_1^\star(s_i^l,s_{ik}^u)] \
        \cr& +(1-w^{u}_{ik})\log[\phi_0^\star(s_i^l,s_{ik}^u)]),
    \end{split}
    \label{eq:triple}
\end{equation}
\noindent
where $w^{l}_i$ denotes the point-wise probability for the collected triple $s_i^l$ being positive, $w^{u}_{ik}$ denotes the probability for the uncollected triple $s^u_{ik}$ being positive. Based on the definition, the posterior probability of each collected/uncollected triple being positive can be computed as:
 \begin{equation}
    \small
    \Tilde{w}^{l}_i = \frac{\beta\phi_1^l(s^l_i)}{\beta\phi_1^l(s^l_{i}) + (1-\beta)\phi_0^l(s^l_{i})},
\end{equation}
\begin{equation}
    \small
    \Tilde{w}^{u}_{ik} = \frac{\alpha\phi_1^u(s^u_{ik})}{\alpha\phi_1^u(s^u_{ik}) + (1-\alpha)\phi_0^u(s^u_{ik})}.
\end{equation}

To increase model expression ability, instead of forcing $\mathbf{W}^L = \mathbf{\Tilde{W}}^L$ and $\mathbf{W}^U = \mathbf{\Tilde{W}}^U$, we set $\mathbf{W}^L$ and $\mathbf{W}^U$ as free parameters and utilize the term $\mathcal{L}_{KL} = \mathbb{KL}(\mathbf{W}^L\|\mathbf{\Tilde{W}}^L) + \mathbb{KL}(\mathbf{W}^U\|\mathbf{\Tilde{W}}^U)$ to regularize the difference. Finally, based on Eq.~\eqref{eq:llc}, the training objective is formalized as follows:
\begin{equation}
    \small
    \underset{\Theta}{\min}~\mathcal{L} = \underset{\Theta}{\min}~\mathcal{L}_{triple} + \mathcal{L}_{KL} + \mathcal{L}_{reg},
    \label{eq:obj}
\end{equation}
\noindent
where $\mathcal{L}_{reg} = \|\mathbf{W}^L\|_1 + \|\mathbf{W}^U\|_1$ can be viewed as a normalization term. Considering the sparsity property of real-world graphs, $\mathcal{L}_{reg}$ penalizes the posterior estimation that there are too many true positive facts on KG.

\begin{figure}[t]
    \centering
    \includegraphics[width = 1.\linewidth]{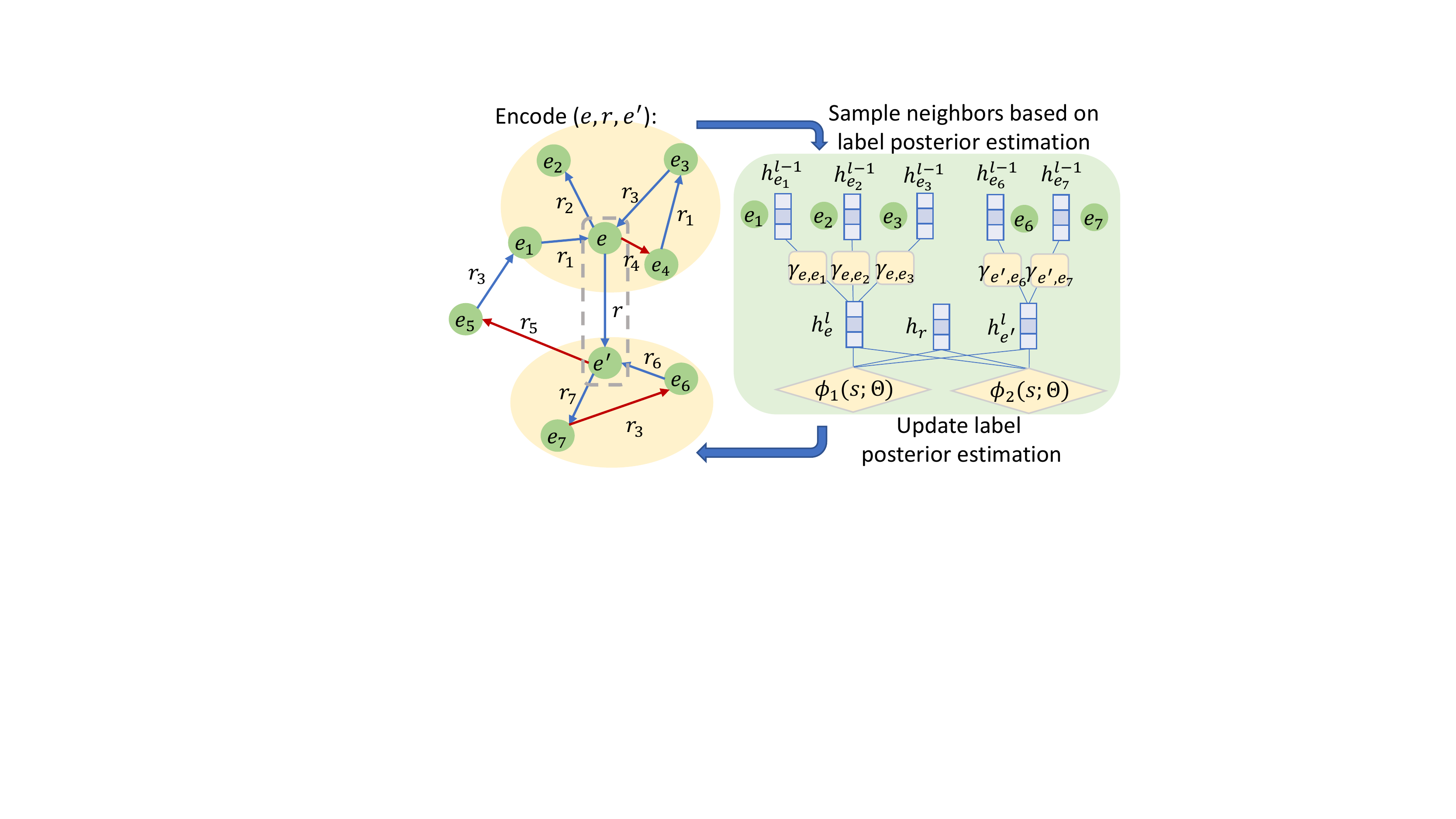}
    \caption{The label posterior-aware encoder. Red line denotes the detected false positive facts based on posterior, which are excluded during neighbor sampling.}
    \label{fig:encoder}
\end{figure}

\subsection{Label Posterior-aware Encoder}
\label{sec:encoder}
We then introduce the encoder and the score functions $\psi_1(s;\Theta)$ and $\psi_0(s;\Theta)$ to measure the probability of positive/negative triples being collected, as shown in Figure~\ref{fig:encoder}. Recent work~\cite{R-GCN,ConvE,ConvKB,CompGCN} has shown that integrating information from neighbors to represent entities engenders better reasoning performance. However, the message-passing mechanism is vulnerable to the false positive issue, as noise can be integrated via a link induced by a false positive fact. In light of this, we propose a label posterior-aware encoder to improve the quality of representations.

We represent each entity $e\in\mathcal{E}$ and each relation $r\in\mathcal{R}$ into a $d$-dimensional latent space: $\mathbf{h}_{e}, \mathbf{h}_r \in \mathbb{R}^d$. To encode more information in $\mathbf{h}_{e}$, we first construct a neighbor set $\mathcal{N}_{e}$ induced by the positive facts related to entity $e$.
The latest label posterior for collected facts $\mathbf{\Tilde{W}}^L$ naturally serves this purpose, as it indicates the underlying correctness for each collected fact.

Therefore, for each entity $e$, we first sort the related facts by label posterior $\mathbf{\Tilde{W}}^L$ and construct the neighbor set $\mathcal{N}_e(\mathbf{\Tilde{W}}^L) = \{(e_i, r_i)\}$ from the top facts. Then the encoder attentively aggregates information from the collected neighbors, where the attention weights take neighbor features, relation features into account. Specifically:
\begin{equation}
    \small
    \mathbf{h}^l_{e} = \mathbf{h}^{l-1}_{e} + \sigma\left(\sum_{(e_i, r_i) \in \mathcal{N}_{e}(\mathbf{\Tilde{W}}^L)} \gamma^l_{e,e_i} \left(\mathbf{h}_{e_i}^{l-1} \mathbf{M}\right)\right),
\end{equation}
\noindent where $l$ denotes the layer number, $\sigma(\cdot)$ denotes the activation function, $\gamma^l_{e,e_i}$ denotes the attention weight of entity $e_i$ to the represented entity $e$, and $\mathbf{M}$ is the trainable transformation matrix. The attention weight $\gamma^l_{e,e_i}$ is supposed to be aware of entity feature and topology feature induced by relations. We design the attention weight $\gamma^l_{e,e_i}$ as follows:
\begin{equation}
    \small
      \gamma^l_{e,e_i} = \frac{\exp(q^l_{e,e_i})}{\underset{\mathcal{N}_{e}(\mathbf{\Tilde{W}}^L)}{\sum}\exp(q^l_{e,e_k})}, ~~ q^l_{e,e_k} = \mathbf{a}\left(\mathbf{h}^{l-1}_{e} \| \mathbf{h}^{l-1}_{e_k} \| \mathbf{h}_{r_k}\right),  
\end{equation}
\noindent
where $q^l_{e,e_k}$ measures the pairwise importance from neighbor $e_k$ by considering the entity embedding, neighbor embedding, and relation embedding, $\mathbf{a}\in\mathbb{R}^{3d}$ is a shared parameter in the attention. 


To measure the collection probability for positive/negative triples, we utilize two multilayer perception (MLP) to approximate score function $\psi_1(s;\Theta)$ and $\psi_0(s;\Theta)$. Specifically, for each triple $s = (e_h, r, e_t)$:
\begin{equation}
    \small
    \psi_1(s;\Theta) = \text{MLP}_1(\mathbf{h}_s), ~ \psi_0(s;\Theta) = \text{MLP}_0(\mathbf{h}_s),
\end{equation}
\noindent
where the MLP input $\mathbf{h}_s = [\mathbf{h}_{e_h}^l\|\mathbf{h}_{r}\|\mathbf{h}_{e_t}^l]$ concatenates entity embeddings and relation embedding.

\subsection{Self-Training Strategy}
\label{sec:st}
The latest label posterior $\mathbf{\Tilde{W}}^L$ and $\mathbf{\Tilde{W}}^U$ is further utilized in a self-training strategy to enhance speculative reasoning. First, the latest posterior estimation $\mathbf{\Tilde{W}}^L$ for collected links updates neighbor sets to gradually prevent the encoder effects by false positive links. Moreover, the latest estimation $\mathbf{\Tilde{W}}^U$ for uncollected facts enables us to continuously sample unlabeled triplets with high label posterior to cover positive samples in the unlabeled set to the greatest extent. For each labeled triple $s_i^l = (e_h, r, e_t)$, we construct $K$ unlabeled triples $s_{ik}^u$ by replacing the head and tail respectively with other entities: $s_{ik}^u = (e_h, r, e_k^{-}) \ \text{or} \ (e_k^{-}, r, e_t)$, where $e_k^{-}$ is the selected entity that ensures $s_{ik}^u \notin \mathcal{S}^L$. Such selection is performed by ranking the corresponding label posterior $\Tilde{w}_{ik}^u$. The updates of neighbor sets and unlabeled triples based on label posterior are nested with parameter optimization during model training alternatively. The training of \model is summarized in Algorithm~\ref{al:training}.


\label{sec:opt}
\begin{algorithm}[t]
\caption{Summary of \model.}
\label{al:training}
\small
\KwIn{The collected triple set $\mathcal{S}^L$.}
\KwOut{The model parameter $\Theta$, predicted triples.}
Construct the uncollected triple set $\mathcal{S}^U$ randomly; \\
Initialize the model parameter $\Theta$ and the label posterior $\Tilde{\mathbf{W}}^L$ and $\Tilde{\mathbf{W}}^U$ randomly; \\
\For{each training epoch}{
    Construct neighbor set $\mathcal{N}_e(\Tilde{\mathbf{W}}^L)$ by $\Tilde{\mathbf{W}}^L$; \\
    Construct uncollected triple set $\mathcal{S}^U$ by $\Tilde{\mathbf{W}}^U$; \\
    \For{each collected triple $s_i^l\in\mathcal{S}^L$}{
        Collect unlabeled triples $\{s^u_{ik}\}_{k=1}^K$;\\
        Calculate $\phi_y^l(s^l_i)$ by Eq.~\eqref{eq:prob};\\
        Calculate each $\phi_y^\star(s^l_i, s^u_{ik})$ by Eq.~\eqref{eq:star_0};\\
    }
    Calculate the total loss $\mathcal{L}$ by Eq.~\eqref{eq:obj};\\
    Optimize model parameter: $\Theta = \Theta - \frac{\partial\mathbf{L}}{\partial\Theta}$;\\
    Update label posterior $\Tilde{\mathbf{W}}^L$ and $\Tilde{\mathbf{W}}^U$;\\
}
\end{algorithm}
\vspace{-3mm}

\section{Experiment}

\subsection{Experimental Setup}

\noindent\textbf{Dataset}.
We evaluate \model mainly on three benchmark datasets: FB15K~\cite{TransE}, FB15k-237~\cite{toutanova-etal-2015-representing}, and WN18~\cite{TransE} and one Twitter data we collected, which describes user interaction information towards tweets and hashtags. Table~\ref{tb:data} summarizes the dataset statistics.

To better fit the real scenario for speculative reasoning, we randomly modify links on KG to simulate more false negative/positive cases. We modify a specific amount of positive/negative links (the ratio of the modified links is defined as perturbation rate, i.e., \textit{ptb\_rate}) by flipping. $90\%$ of them are the removed positive links to simulate false negative cases and the remaining $10\%$ are the added negative links to simulate false positive cases. More details about datasets and the data perturbation process can be found in Appendix~\ref{ap:data}.


\begin{table}[t]
\caption{The statistics of the datasets.}
\label{tb:data}
\small
\centering
\resizebox{1.0\linewidth}{!}{
\fontsize{8.5}{11}\selectfont
\begin{tabular}{c|c|c|c|c|c|c}
\toprule
\textbf{{\em ptb\_rate\/}}& \textbf{Dataset}   & \textbf{\textbf{$|\mathcal{E}|$}} & \textbf{\textbf{$|\mathcal{R}|$}} & \textbf{\#Train} & \textbf{\#Valid} & \textbf{\#Test} \\ \hline
\multirow{4}{*}{0.1} & \textbf{FB15K} & 14,951 & 1,345 & 340,968 & 146,129 & 59,071 \\
&\textbf{FB15K-237} & 14,541 & 237 & 184,803 & 79,201 & 20,466 \\
&\textbf{WN18} & 40,943 & 18 & 92,428 & 39,612 & 5,000 \\
&\textbf{Twitter} & 17,839 & 2 & 282,233 & 120,956 & 110,456 \\ \hline
\multirow{4}{*}{0.3} & \textbf{FB15K}     & 14,951              & 1,345               & 276,940          & 118,688          & 59,071          \\
&\textbf{FB15K-237} & 14,541              & 237                 & 149,229          & 63,954           & 20,466          \\
&\textbf{WN18}      & 40,943              & 18                  & 72,462           & 31,055           & 5,000           \\
&\textbf{Twitter}   & 17,839              & 2                   & 232,748          & 99,749           & 110,456         \\
\hline
\multirow{4}{*}{0.5} & \textbf{FB15K} & 14,951 & 1,345 & 213,380 & 91,448 & 59,071 \\
&\textbf{FB15K-237} & 14,541 & 237 & 113,772 & 48,759 & 20,466 \\
&\textbf{WN18} & 40,943 & 18 & 52,707 & 22,588 & 5,000 \\
&\textbf{Twitter} & 17,839 & 2 & 183,263 & 78,540 & 110,456 \\ 
\hline
\multirow{4}{*}{0.7} & \textbf{FB15K} & 14,951 & 1,345 & 150,485 & 64,493 & 59,071 \\
&\textbf{FB15K-237} & 14,541 & 237 & 78,531 & 33,656 & 20,466 \\
&\textbf{WN18} & 40,943 & 18 & 34,984 & 14,993 & 5,000 \\
&\textbf{Twitter} & 17,839 & 2 & 133,778 & 57,333 & 110,456 \\  
\bottomrule
\end{tabular}}
\end{table}

\begin{table*}[t]
\caption{Overall performance on noisy and incomplete graphs, with {\em ptb\_rate = 0.3\/}. Average results on $5$ independent runs are reported. $*$ indicates the statistically significant results over baselines, with $p$-value $<0.01$. The best results are in boldface, and the strongest baseline performance is underlined.}
\label{tb:main_0.3}
\small
\centering
\resizebox{0.95\textwidth}{!}{
\fontsize{8.5}{11}\selectfont
\begin{tabular}{c|cccc|cccc|cccc}
\toprule
\textbf{Dataset} & \multicolumn{4}{c|}{\textbf{FB15K}} & \multicolumn{4}{c|}{\textbf{FB15K-237}} & \multicolumn{4}{c}{\textbf{WN18}} \\ \hline
\textbf{Metrics} & \textbf{MRR} & \textbf{H@10} & \textbf{H@3} & \textbf{H@1} & \textbf{MRR} & \textbf{H@10} & \textbf{H@3} & \textbf{H@1} & \textbf{MRR} & \textbf{H@10} & \textbf{H@3} & \textbf{H@1} \\ \hline
\multicolumn{13}{c}{ \textit{Knowledge graph embedding methods}} \\ \hline
\textbf{TransE} & 0.336 & 0.603 & 0.425 & 0.189 & 0.196 & 0.394 & 0.236 & 0.094 & 0.229 & 0.481 & 0.416 & 0.030 \\
\textbf{TransR} & 0.314 & 0.579 & 0.397 & 0.170 & 0.184 & 0.359 & 0.211 & 0.098 & 0.229 & 0.480 & 0.408 & 0.035 \\
\textbf{DistMult} & 0.408 & 0.627 & 0.463 & 0.296 & 0.240 & 0.407 & 0.262 & 0.158 & 0.397 & 0.518 & 0.453 & 0.320 \\
\textbf{ComplEx} & 0.396 & 0.616 & 0.451 & 0.284 & 0.238 & 0.411 & 0.262 & 0.154 & 0.448 & 0.526 & 0.475 & 0.403 \\
\textbf{RotatE} & {\ul 0.431} & {\ul 0.636} & {\ul 0.489} & {\ul 0.323} & {\ul 0.255} & {\ul 0.433} & {\ul 0.280} & 0.169 & 0.446 & 0.524 & 0.474 & 0.400 \\ \hline
\multicolumn{13}{c}{ \textit{Graph neural network methods on KG}} \\\hline
\textbf{RGCN} & 0.154 & 0.307 & 0.164 & 0.078 & 0.141 & 0.276 & 0.145 & 0.075 & 0.362 & 0.464 & 0.412 & 0.300 \\
\textbf{CompGCN} & 0.409 & 0.631 & 0.465 & 0.294 & 0.253 & 0.422 & 0.275 & {\ul 0.171} & 0.445 & 0.522 & 0.471 & 0.400 \\ \hline
\multicolumn{13}{c}{ \textit{Uncertain knowledge graph embedding method}} \\\hline
\textbf{UKGE} & 0.311 & 0.556 & 0.337 & 0.189 & 0.172 & 0.233 & 0.128 & 0.081 & 0.241 & 0.447 & 0.309 & 0.119 \\ \hline
\multicolumn{13}{c}{ \textit{Negative sampling methods}} \\\hline
\textbf{NSCaching} & 0.371 & 0.576 & 0.424 & 0.265 & 0.190 & 0.329 & 0.208 & 0.121 & 0.306 & 0.401 & 0.334 & 0.255 \\
\textbf{SANS} & 0.372 & 0.599 & 0.434 & 0.252 & 0.243 & 0.416 & 0.267 & 0.158 & {\ul 0.453} & {\ul 0.528} & {\ul 0.479} & {\ul 0.409} \\ \hline
\multicolumn{13}{c}{ \textit{Positive-Unlabeled learning methods on KG}} \\\hline
\textbf{PUDA} & 0.403 & 0.623 & 0.458 & 0.291 & 0.234 & 0.394 & 0.255 & 0.156 & 0.382 & 0.499 & 0.444 & 0.306 \\ \hline
\textbf{nPUGraph} & \textbf{0.486*} & \textbf{0.718*} & \textbf{0.534*} & \textbf{0.342*} & \textbf{0.287*} & \textbf{0.481*} & \textbf{0.315*} & \textbf{0.191*} & \textbf{0.493*} & \textbf{0.582*} & \textbf{0.519*} & \textbf{0.442*} \\ \hline
\rowcolor{LightCyan} \textbf{Gains ($\%$)} & \textit{12.7} & \textit{12.8} & \textit{9.2} & \textit{5.9} & \textit{12.6} & \textit{11.2} & \textit{12.5} & \textit{11.4} & \textit{8.9} & \textit{10.3} & \textit{8.3} & \textit{8.0} \\ \bottomrule
\end{tabular}}
\end{table*}

\noindent\textbf{Baselines}.
We compare to eleven state-of-the-art baselines: 1) KG embedding methods: \textbf{TransE}~\cite{TransE}, \textbf{TransR}~\cite{TransR}, \textbf{DistMult}~\cite{DistMult}, \textbf{ComplEX}~\cite{ComplEX}, and \textbf{RotatE}~\cite{RotatE}; 2) GNN methods on KG: \textbf{RGCN}~\cite{R-GCN} and \textbf{CompGCN}~\cite{CompGCN}; 3) Uncertain KG reasoning: \textbf{UKGE}~\cite{UKGE}; 4) Negative sampling methods: \textbf{NSCaching}~\cite{NSCaching} and \textbf{SANS}~\cite{SANS}; 5) PU learning on KG: \textbf{PUDA}~\cite{PUDA}. More details can be found in Appendix~\ref{ap:baseline}.



\noindent\textbf{Evaluation and Implementation}.
For each $(e_h, r, e_?)$ or $(e_?, r, e_t)$, we rank all entities at the missing position in triples, and adopt filtered mean reciprocal rank ({\em MRR}) and filtered Hits at $\{1,3,10\}$ as evaluation metrics~\cite{TransE}. More implementation details of baselines and \model can be found in Appendix~\ref{ap:implementation}.

\subsection{Main Results}
We first discuss the model performance on noisy and incomplete graphs, with {\em ptb\_rate = 0.3\/}, as shown in Table~\ref{tb:main_0.3}. \model achieves consistently better results than all baseline models, with $10.3\%$ relative improvement on average. Specifically, conventional KGE and GNN-based methods produce unsatisfying performance, as they ignore the false negative/positive issues during model training. In some cases, GNN-based ways are worse, as the message-passing mechanism is more vulnerable to false positive links. As expected, the performance of uncertain knowledge graph embedding model~\cite{UKGE} is much worse when there are no available uncertainty scores for model training. SANS and PUDA generate competitive results in some cases, as their negative sampling strategy and PU learning objective can respectively mitigate the false negative/positive issues to some extent. Table~\ref{tb:main_0.3} demonstrates the superiority of \model which addresses the false negative/positive issues simultaneously. For space limitations, we report and discuss the model performance on Twitter data in Table~\ref{tb:twitter} in Appendix~\ref{ap:twitter}.

\begin{figure}
    \centering
    \includegraphics[width = 0.95\linewidth]{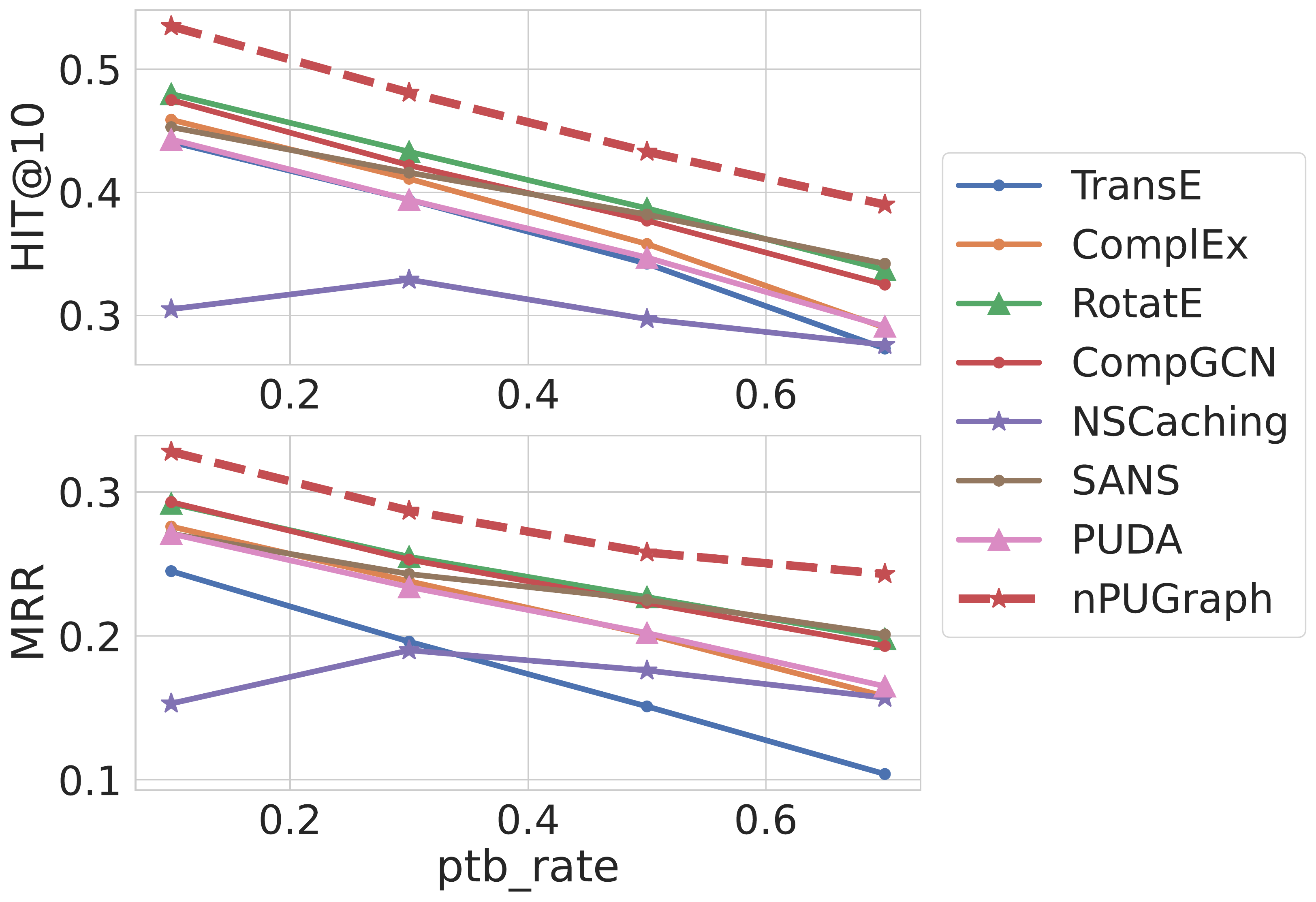}
    \caption{Performance with respect to various {\em ptb\_rate \/}, i.e., different degrees of noise and incompleteness, on \textit{FB15K-237}. \model exhibits impressive robustness against false negative/positive issues.}
    \label{fig:FB15K-237}
\end{figure}

\subsection{Experiments under Various Degrees of Noise and Incompleteness}
We investigate the performance of baseline models and \model under different degrees of noise and incompleteness. Figure~\ref{fig:FB15K-237} reports the performance under various {\em ptb\_rate \/}, from 0.1 to 0.7, where higher {\em ptb\_rate \/} means more links are perturbed as false positive/negative cases. Full results are included in Appendix~\ref{ap:allresult}. The performance degrades as the {\em ptb\_rate \/} increases for all in most cases, demonstrating that the false negative/positive issues significantly affect the reasoning performance. However, \model manages to achieve the best performance in all cases. Notably, the relative improvements are more significant under higher {\em ptb\_rate \/}. 

\begin{table}[t]
\caption{Ablation Studies.}
\label{tb:ablation}
\centering
\scriptsize
\resizebox{1.0\linewidth}{!}{
    \fontsize{8.5}{11}\selectfont
\begin{tabular}{c|cc|cc|c}
\toprule
\multicolumn{1}{c|}{\textbf{Dataset}} & \multicolumn{2}{c|}{\textbf{FB15K}} & \multicolumn{2}{c|}{\textbf{FB15K-237}} & \textbf{Gains} \\ \hline
\multicolumn{1}{c|}{\textbf{Ablations}} & \textbf{MRR} & \textbf{H@10} & \textbf{MRR} & \textbf{H@10} & \textbf{\%} \\ \hline
\textbf{nPUGraph w/o nPU} & 0.401 & 0.619 & 0.230 & 0.407 & -20.0 \\
\textbf{nPUGraph w/o  LP-Encoder} & 0.457 & 0.681 & 0.261 & 0.459 & -6.6 \\
\textbf{nPUGraph w/o  Self-Training} & 0.471 & 0.704 & 0.276 & 0.461 & -3.4 \\ \hline
\textbf{nPUGraph} & 0.486 & 0.718 & 0.287 & 0.481 & - \\ \bottomrule
\end{tabular}}
\vspace{-5mm}
\end{table}

\subsection{Model Analysis}
\noindent \textbf{Ablation Study}.
We evaluate performance improvements brought by the \model framework by following ablations: 1) {\bf \model w/o nPU} is trained without the noisy Positive-Unlabeled framework, which instead utilizes the margin loss for model training; 2) {\bf \model w/o LP-Encoder} eliminates the label posterior-aware encoder (LP-Encoder), which aggregates information from all neighbors instead of the sampled neighbors; 3) {\bf \model w/o Self-Training} is trained without the proposed self-training algorithm.

We report {\em MRR\/} and {\em Hit@10\/} over FB15K and FB15K-237 data, as shown in Table~\ref{tb:ablation}. As we can see, training the encoder without the proposed noisy Positive-Unlabeled framework will cause the performance drop, as this variant ignores the false negative/positive issues. Removing the label posterior-based neighbor sampling in the encoder will also cause performance degradation, as the information aggregation no longer distinguishes between true and false links. Such a variant can be easily influenced by the existence of false positive facts. Moreover, the last ablation result shows if the training process is further equipped with the self-training strategy, the performance will be enhanced, which verifies its effectiveness to select informative unlabeled samples for model training.

\noindent\textbf{The Effect of PU Triple Construction}.
We then investigate the effect of PU Triple Construction on model performance, by varying different sizes of unlabeled samples from $10$ to $50$. Figure~\ref{fig:PUSize} shows that the performance improves as the number of unlabeled samples increases. Because more unlabeled samples can cover more false negative cases for model training. The training time grows linearly.

\noindent\textbf{The Effect of Positive Class Prior $\alpha$}.
The positive class prior $\alpha$ and true positive ratio $\beta$ are two important hyperparameters. While $\beta$ has a clear definition from real-world data, the specific value of $\alpha$ is unknown in advance. Figure~\ref{fig:classprior} shows the model performance w.r.t. different values of $\alpha$ by grid search. Reasoning performance fluctuates a bit with different values of $\alpha$ since incorrect prior knowledge of $\alpha$ can bias the label posterior estimation and thus hurt the performance.  

\begin{figure}[t]
    \centering
    \begin{subfigure}[b]{0.55\linewidth}
    \centering
    \includegraphics[width = \linewidth]{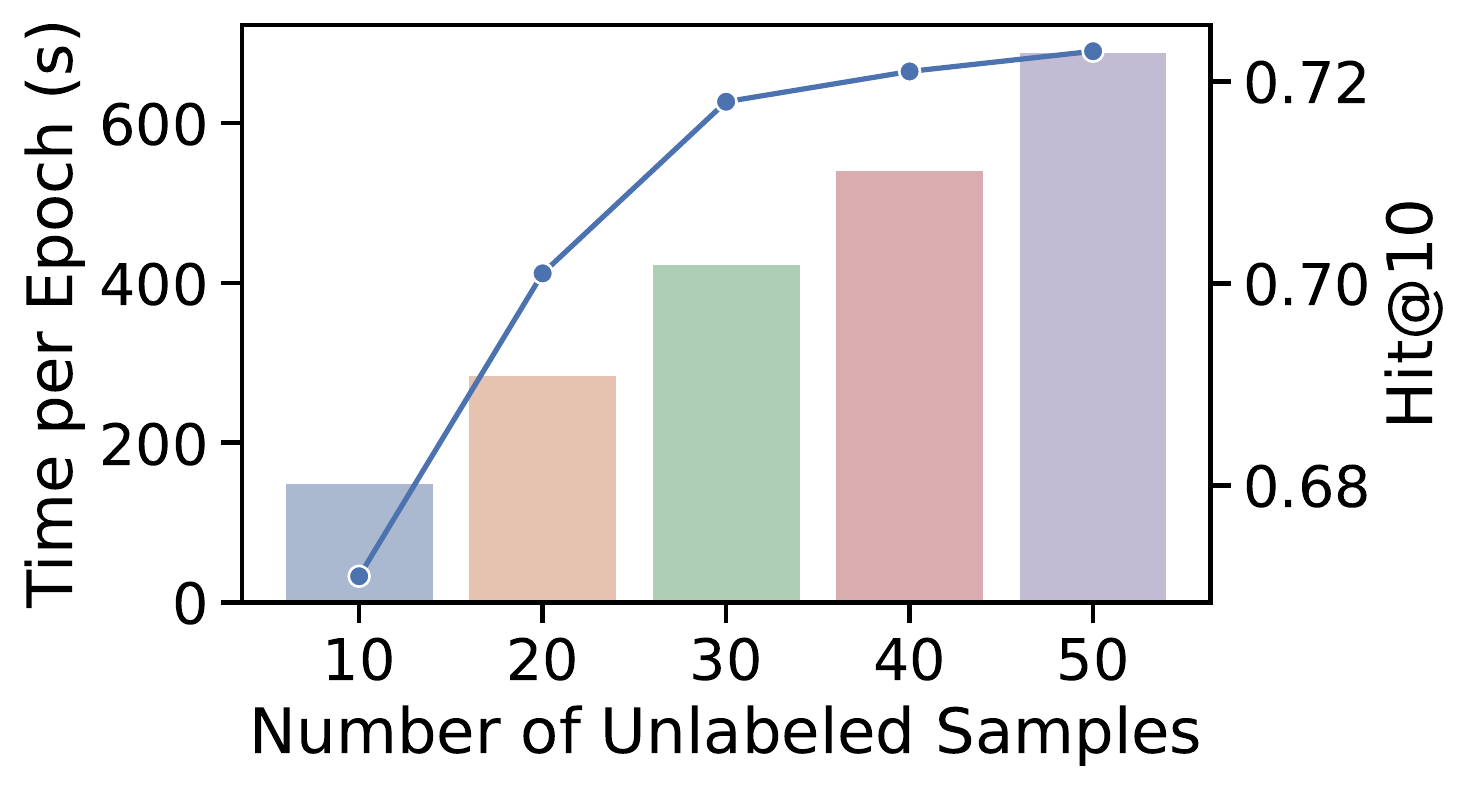}
    \caption{\#unlabeled samples.}
    \label{fig:PUSize}
    \end{subfigure}
    \begin{subfigure}[b]{0.43\linewidth}
    \centering
    \includegraphics[width = \linewidth]{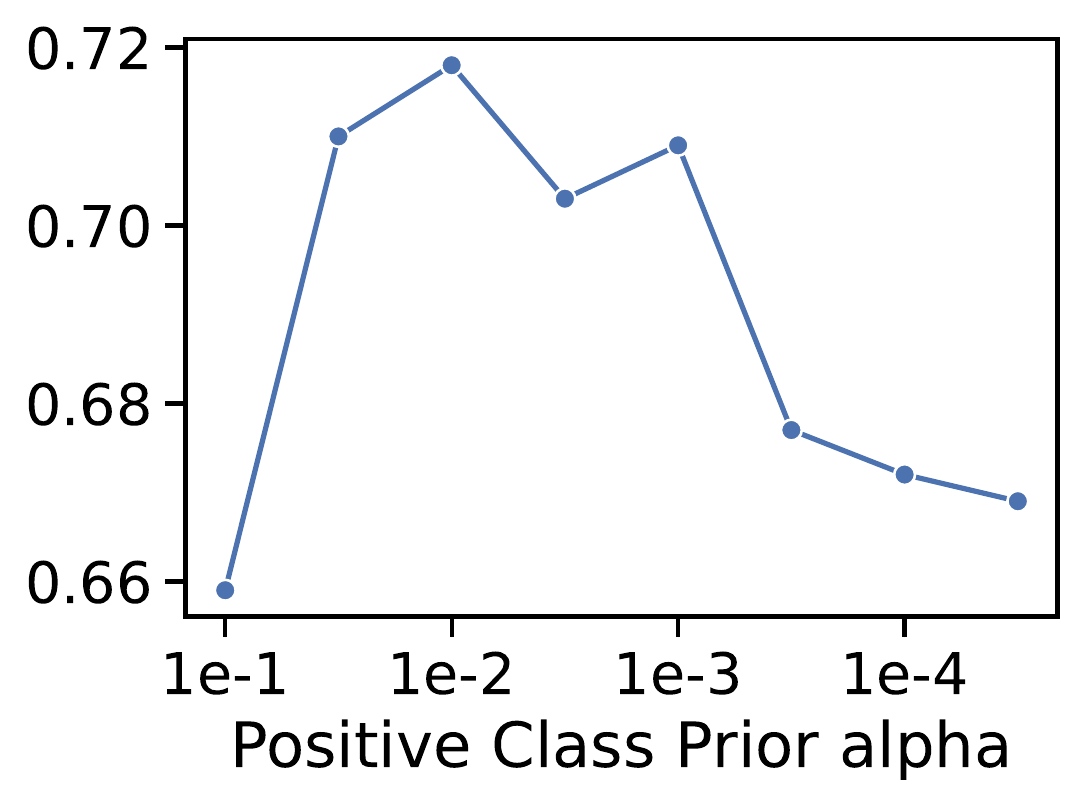}
    \caption{Class prior $\alpha$.}
    \label{fig:classprior}
    \end{subfigure}
    \caption{Model analysis w.r.t. the number of unlabeled samples and positive class prior $\alpha$.}
    \label{fig:analysis}
\end{figure}

\section{Related Work}

\noindent \textbf{Knowledge Graph Reasoning}.
Knowledge graph reasoning (KGR) aims to predict missing facts to automatically complete KG, including one-hop reasoning~\cite{TransE} and multi-hop reasoning~\cite{KBQA}. It has facilitated a wide spectrum of knowledge-intensive applications~\cite{wang2018ace,AND,KBQA,KG4Social,misinfo,polarization,CMH,bright,DKGA,infovgae,wang2022rete}.To set the scope, we primarily focus on one-hop reasoning and are particularly interested in predicting missing entities in a partial fact. Knowledge graph embeddings achieve state-of-the-art performance~\cite{TransE,TransR,DistMult,ComplEX,RotatE,TA-DistMult,wang2022learning,wang2023mpkd}. Recently, graph neural networks (GNN) have been incorporated to enhance representation learning by aggregating neighborhood information~\cite{R-GCN,ConvE,ConvKB,CompGCN,RE-GCN}. However, most approaches significantly degrade when KG are largely incomplete and contain certain errors~\cite{noisyKG}, as they ignore the false negative/positive issues. Recent attempts on uncertain KG~\cite{UKGE,GTransE,boxembedding} measure the uncertainty score for facts, which can detect false negative/positive samples. However, they explicitly require the ground truth uncertainty scores for model training, which are usually unavailable in practice. Various negative sampling strategies have been explored to sample informative negative triples to facilitate model training~\cite{KBGAN,NSCaching,SANS}. However, they cannot detect false negative/positive facts. We aim to mitigate the false negative/positive issues and enable the automatic detection of false negative/positive facts during model training. 


\noindent \textbf{Positive-Unlabeled Learning}. Positive-Unlabeled (PU) learning is a learning paradigm for training a model that only has access to positive and unlabeled data, where unlabeled data includes both positive and negative samples~\cite{PULearning,PUSurvey}. PU learning roughly includes (1) two-step solutions~\cite{PU_two_step,nPU}; (2) methods that consider the unlabeled samples as negative samples with label noise~\cite{PU_noisy}; (3) unbiased risk estimation methods~\cite{PULearning,PUDA}. Recent work further studies the setting that there exists label noise in the observed positive samples~\cite{nPU}. We formulate the KGR task on noisy and incomplete KG as a noisy Positive-Unlabeled learning problem and propose a variational framework for it, which relates to two-step solutions and unbiased risk estimation methods.

\section{Conclusion}
\label{sec:conclusion}

We studied speculative KG reasoning based on sparse and unreliable observations, which contains both \textit{false negative issue} and \textit{false positive issue}. We formulated the task as a noisy Positive-Unlabeled learning problem and proposed a variational framework \model to jointly update model parameters and estimate the posterior likelihood of collected/uncollected facts being true or false, where the underlying correctness is viewed as latent variables. During the training process, a label posterior-aware encoder and a self-training strategy were proposed to further address the false positive/negative issues. We found label posterior estimation plays an important role in moving toward speculative KG reasoning in reality, and the estimation can be fulfilled by optimizing an alternative objective without additional cost. Extensive experiments verified the effectiveness of \model on both benchmark KGs and Twitter interaction data with various degrees of data perturbations. 


\section*{Limitations}
\label{sec:limitation}
There are certain limitations that can be concerned for further improvements. First, the posterior inference relies on the prior estimation of positive class prior $\alpha$ and true positive ratio $\beta$. Our experiments show that a data-driven estimation based on end-to-end model training produces worse results than a hyperparameter grid search. An automatic prior estimation is desirable for real-world applications. Moreover, in \model, we approximate the probability of negative/positive facts being collected/uncollected via neural networks, which lacks a degree of interpretability. In the future, we plan to utilize a more explainable random process depending on entity/relation features to model the collection probability distribution. 

\section*{Ethical Impact}
\label{sec:ethical}

 \model neither introduces any social/ethical bias to the model nor amplifies any bias in data. Benchmark KG are publicly available. For Twitter interaction data, we mask all identity and privacy information for users, where only information related to user interactions with tweets and hashtags is presented. Our model is built upon public libraries in PyTorch. We do not foresee any direct social consequences or ethical issues.

 \section*{Acknowledgements}

 The authors would like to thank the anonymous reviewers for their valuable comments and suggestions. Research reported in this paper was sponsored in part by DARPA award HR001121C0165, DARPA award HR00112290105, DoD Basic Research Office award HQ00342110002, the Army Research Laboratory under Cooperative Agreement W911NF-17-20196. It was also supported in part by ACE, one of the seven centers
in JUMP 2.0, a Semiconductor Research Corporation (SRC)
program sponsored by DARPA. The views and conclusions contained in this document are those of the author(s) and should not be interpreted as representing the official policies of DARPA and DoD Basic Research Office or the Army Research Laboratory. The US government is authorized to reproduce and distribute reprints for government purposes notwithstanding any copyright notation hereon.
\bibliographystyle{acl_natbib}
\bibliography{main.bib}
\newpage
\appendix
\section{Appendix}

\subsection{Theorem Proof}
\label{ap:proof}

\begin{theorem}
\label{the:llc_ap}

The log-likelihood of the complete data $\log p(\mathbf{S})$ is lower bounded as follows:
\begin{equation}
\scriptsize
    \begin{split}
        \cr &\log p(\mathbf{S})\geq \underset{q(\mathbf{Y})}{\mathbb{E}} \left[\log p(\mathbf{S}|\mathbf{Y})\right] - \mathbb{KL}(q(\mathbf{Y})\| p(\mathbf{Y})) \
        \cr& = \underset{s^l\in\mathcal{S}^L}{\mathbb{E}}\left[w^{l}\log[\phi_1^l(s^l)]+(1-w^{l})\log[\phi_0^l(s^l)]\right] \
        \cr& + \underset{s^u\in\mathcal{S}^U}{\mathbb{E}}\left[(w^{u}\log[\phi_1^u(s^u)]+(1-w^{u})\log[\phi_0^u(s^u)]\right] \
        \cr& - \mathbb{KL}(\mathbf{W}^U\|\mathbf{\Tilde{W}}^U) - \mathbb{KL}(\mathbf{W}^L\|\mathbf{\Tilde{W}}^L) - \frac{\|\mathbf{W}^L\|_1}{|\mathcal{S}^L|} - \frac{\|\mathbf{W}^U\|_1}{|\mathcal{S}^U|},
    \end{split}
    \label{eq:llc_ap}
\end{equation}
\noindent
where $\mathbf{S}$ denotes all labeled/unlabeled triples, $\mathbf{Y}$ is the corresponding latent variable indicating the positive/negative labels for triples, $\mathbf{W}^U = \{w^{u}_i\}$ denotes the point-wise probability for the uncollected triples being positive, $\mathbf{W}^L = \{w^{l}_i\}$ denotes the probability for the collected triples being positive.
\end{theorem}

\begin{proof}
Let $\log p(\mathbf{S})$ denote the log-likelihood of all potential triples being collected in the KG or not, $\mathbf{Y}$ denote the corresponding latent variable indicating the positive/negative labels. We aim to infer the label posterior $p(\mathbf{Y}|\mathbf{S})$, which can be approximated by $q(\mathbf{Y}|\mathbf{S})$. We therefore are interested at the difference between the two, measured by the Kullback–Leibler (KL) divergence as follows:

\begin{equation}
\scriptsize
    \mathbb{KL}(q(\mathbf{Y}|\mathbf{S})\|p(\mathbf{Y}|\mathbf{S})) = -\underset{q(\mathbf{Y}|\mathbf{S})}{\mathbb{E}}\log\left(\frac{p(\mathbf{S}|\mathbf{Y})p(\mathbf{Y})}{q(\mathbf{Y}|\mathbf{S})}\right) + \log p(\mathbf{S}),
\end{equation}
as KL divergence is positive, we derive the lower bound of the log-likelihood as follows:

\begin{equation}
\scriptsize
    \begin{split}
        \cr& \log p(\mathbf{S}) \geq \underset{q(\mathbf{Y}|\mathbf{S})}{\mathbb{E}}\log\left(\frac{p(\mathbf{S}|\mathbf{Y})p(\mathbf{Y})}{q(\mathbf{Y}|\mathbf{S})}\right) \
        \cr& \geq \underset{q(\mathbf{Y}|\mathbf{S})}{\mathbb{E}}\log p(\mathbf{S}|\mathbf{Y}) - \mathbb{KL}(q(\mathbf{Y}|\mathbf{S})\|p(\mathbf{Y})) - \underset{p(\mathbf{S})}{\mathbb{E}}q(\mathbf{Y}|\mathbf{S}),
    \end{split}
    \label{eq:derive}
\end{equation}
\noindent
which consists of three terms: triple collection probability measure, KL term and regularization of label posterior (positive). We discuss each term in detail.

Recall that the distribution of labeled triples can be represented as follows:

\begin{equation}
    \scriptsize
    s^l \sim \beta \phi_1^l(s^l) + (1-\beta) \phi_0^l(s^l),
\end{equation}
\noindent
where $\phi_y^l$ denotes the probability of being collected over triple space $\mathcal{S}$ for the positive class ($y=1$) and negative class ($y=0$), and $\beta\in [0,1)$ denotes the proportion of true positive samples in labeled data. Similarly, considering the existence of unlabeled positive triples, the distribution of unlabeled samples can be represented as follows:
\begin{equation}
    \scriptsize
    s^u \sim \alpha \phi_1^u(s^u) + (1-\alpha) \phi_0^u(s^u),
\end{equation}
\noindent
where $\phi_y^u = 1- \phi_y^l$ denotes the probability of being uncollected, $\alpha\in [0,1)$ is the class prior or the proportion of positive samples in unlabeled data. Based on that, the first term can be detailed as follows:

\begin{equation}
\scriptsize 
    \begin{split}
         \cr& \underset{q(\mathbf{Y}|\mathbf{S})}{\mathbb{E}}\log p(\mathbf{S}|\mathbf{Y}) = \underset{s^l\in\mathcal{S}^L}{\mathbb{E}} \underset{y \in \{0,1\}}{\mathbb{E}} q(y|s^l) \log p(s^l|y) \
         \cr&+ \underset{s^u\in\mathcal{S}^U}{\mathbb{E}} \underset{y \in \{0,1\}}{\mathbb{E}} q(y|s^l) \log p(s^l|y) \
         \cr& = \underset{s^l\in\mathcal{S}^L}{\mathbb{E}}\left[q(y=1|s^l)\log[p(s^l|y=1))]  + q(y=0|s^l)\log[p(s^l|y=0))]\right] \
        \cr& + \underset{s^u\in\mathcal{S}^U}{\mathbb{E}}\left[q(y=1|s^u)\log[p(s^u|y=1))] + q(y=0|s^u)\log[p(s^u|y=0))]\right] \
        \cr& = \underset{s^l\in\mathcal{S}^L}{\mathbb{E}}\left[w^{l}\log[\phi_1^l(s^l)]+(1-w^{l})\log[\phi_0^l(s^l)]\right] \
        \cr& + \underset{s^u\in\mathcal{S}^U}{\mathbb{E}}\left[(w^{u}\log[\phi_1^u(s^u)]+(1-w^{u})\log[\phi_0^u(s^u)]\right],
    \end{split}
\end{equation}
\noindent
where $\mathbf{W}^U = \{w^{u}\}$ denotes the point-wise probability for the uncollected triples being positive $q(y=1|s^u)$, $\mathbf{W}^L = \{w^{l}\}$ denotes the probability for the collected triples being positive $q(y=1|s^l)$.

We view $\mathbf{W}^U$ and $\mathbf{W}^L$ as free parameters and regularize them by $\mathbf{\Tilde{W}}^U$ and $\mathbf{\Tilde{W}}^L$, which are estimated posterior probability as follows:
 \begin{equation}
    \scriptsize
    \Tilde{w}^{l}_i = \frac{\beta\phi_1^l(s^l_i)}{\beta\phi_1^l(s^l_{i}) + (1-\beta)\phi_0^l(s^l_{i})},
\end{equation}

\begin{equation}
    \scriptsize
    \Tilde{w}^{u}_{ik} = \frac{\alpha\phi_1^u(s^u_{ik})}{\alpha\phi_1^u(s^u_{ik}) + (1-\alpha)\phi_0^u(s^u_{ik})}.
\end{equation}

Therefore, the KL term in Eq.~\eqref{eq:derive} becomes $\mathbb{KL}(\mathbf{W}^U\|\mathbf{\Tilde{W}}^U) + \mathbb{KL}(\mathbf{W}^L\|\mathbf{\Tilde{W}}^L)$.

Last but not least, the third term $\underset{p(\mathbf{S})}{\mathbb{E}}q(\mathbf{Y}|\mathbf{S}) = \|\mathbf{W}^U\|_1/|\mathcal{S}^U| + \|\mathbf{W}^L\|_1/|\mathcal{S}^L|$ regulates the total number of potential triples (including both collected ones and uncollected ones), because of the sparsity nature of graphs. Finally, we derive the lower bound of the log-likelihood, as shown in Eq~\eqref{eq:llc_ap}.
\end{proof}

\subsection{Datasets}
\label{ap:data}

\subsubsection{Dataset Information}

We evaluate our proposed model based on three widely used knowledge graphs and one Twitter ineraction graph:

\begin{itemize}[leftmargin = 15pt]
    \item \textbf{FB15K}~\cite{TransE} is a subset of Freebase~\cite{freebase}, a large database containing general knowledge facts with a variety of relation types;
    \item \textbf{FB15K-237} ~\cite{toutanova-etal-2015-representing} is a reduced version of FB15K, where inverse relations are removed;
    \item \textbf{WN18} ~\cite{TransE} is a subset of WorldNet~\cite{wordnet}, a massive lexical English database that captures semantic relations between words;
    \item \textbf{Twitter} is an interaction graph relevant with \textit{Russo-Ukrainian War}. Data is collected in the Twitter platform from May 1, 2022, to December 25, 2022, which records user-tweet interactions and user-hashtag interactions. Thus, the graph is formed by two relations (user-tweet and user-hashtag) and multiple entities, which can be categorized into three types (user, tweet, and hashtag). Following~\cite{DyDiff-VAE,wang2022rete}, when constructing the graph, thresholds will be selected to remove inactive users, tweets, and hashtags according to their occurrence frequency. We set the thresholds for the user, tweet, and hashtag as 30, 30, and 10, respectively, i.e., if a tweet has fewer than 30 interactions with users, it will be regarded as inactive and removed from the graph. Besides, the extremely frequent users and tweets are deleted as they may be generated by bots. 
\end{itemize}

For all datasets, we first merge the training set and validation set as a whole. Then we simulate noisy and incomplete graphs for the training process by adding various proportions of false negative/positive cases in the merged set. After that, we partition the simulated graphs into new train/valid sets with a ratio of $7:3$ and the test set remains the same. Table~\ref{tb:data} provides an overview of the statistics of the simulated datasets corresponding to various perturbation rates and based graphs.

\subsubsection{Dataset Perturbation}

Data perturbation aims to simulate noisy and incomplete graphs from clean benchmark knowledge graphs. It consists of two aspects: First, to simulate the \textit{false negative issue}, it randomly removes some existing links in a graph, considering the removed links as missing but potentially plausible facts. Second, to simulate the \textit{false positive issue}, it randomly adds spurious links to the graph as unreliable or outdated facts. 

We define perturbation rate, i.e., {\em ptb\_rate \/}, as a proportion of modified edges in a graph to control the amount of removing positive links and adding negative links. For example, if a graph has 100 links and the perturbation rate is 0.5, then we will randomly convert the positivity or negativity of 50 links. Among these 50 modified links, $10\%$ of them will be added and the rest of them will be removed, i.e., we will randomly add 5 negative links and remove 45 positive links to generate a perturbed graph. The perturbed graph can be regarded as noisy and incomplete, leading to significant performance degradation. In our experiments, we set {\em ptb\_rate \/} in a range of $\{0.1, 0.3, 0.5, 0.7\}$. The detailed perturbation process is summarized in Figure~\ref{fig:ptb}.

\begin{figure*}
    \centering
    \includegraphics[width = 0.8\linewidth]{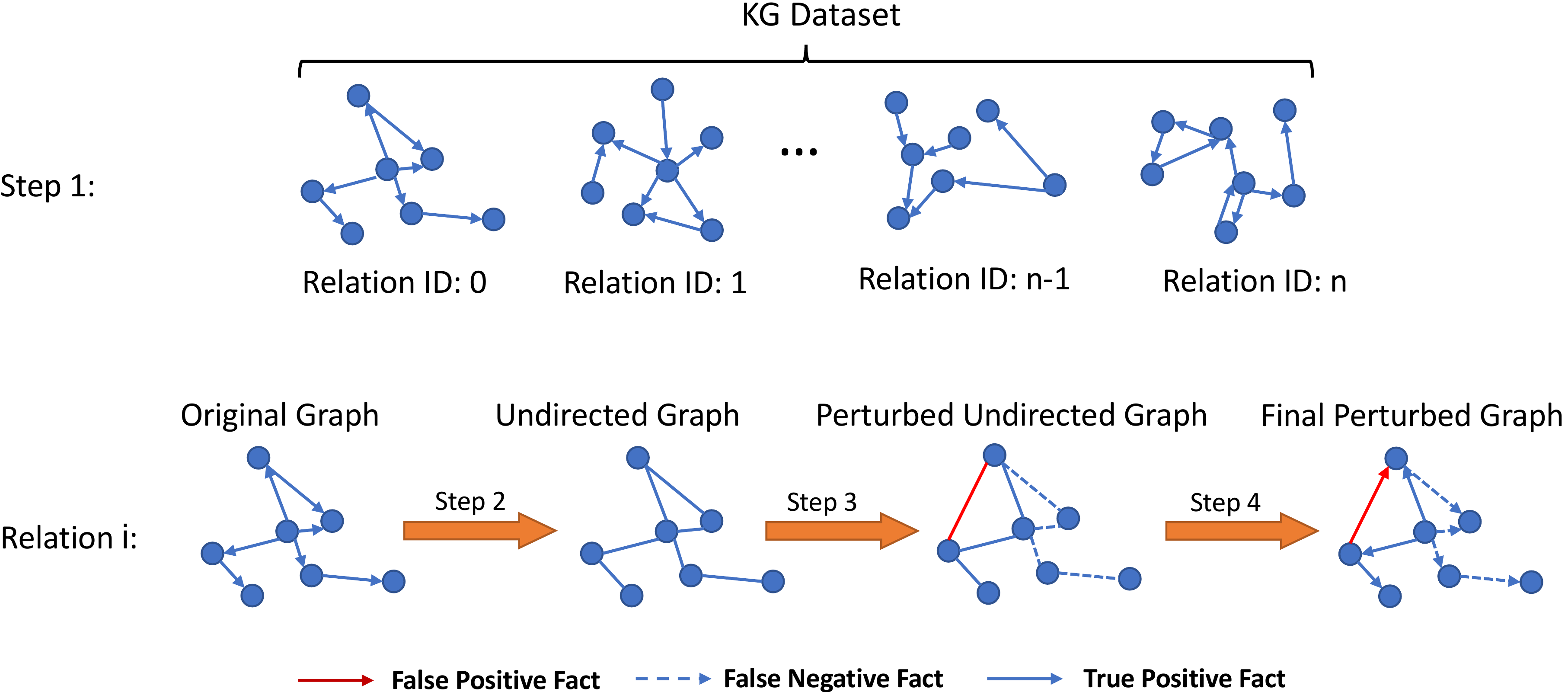}
    \caption{Summary of graph perturbation to construct noisy and incomplete KGs.}
    \label{fig:ptb}
\end{figure*}

\subsection{Baselines}
\label{ap:baseline}
We describe the baseline models utilized in the experiments in detail:
\begin{itemize}[leftmargin = 15pt]
    \item \textbf{TransE}\footnote{The experiments of TransE and TransR are implemented with \url{https://github.com/thunlp/OpenKE}.}~\cite{TransE} is a translation-based embedding model, where both entities and relations are represented as vectors in the latent space. The relation is utilized as a translation operation between the subject and the object entity;
    \item \textbf{TransR}~\cite{TransR}  advances TransE by optimizing modeling of n-n relations, where each entity embedding can be projected to hyperplanes defined by relations;
    \item \textbf{DistMult}\footnote{The experiments of DistMult, ComplEx, and RotatE are implemented with \url{https://github.com/DeepGraphLearning/KnowledgeGraphEmbedding}.}~\cite{DistMult} is a general framework with the bilinear objective for multi-relational learning that unifies most multi-relational embedding models;
    \item \textbf{ComplEx}~\cite{ComplEX} introduces complex embeddings, which can effectively capture asymmetric relations while retaining the efficiency benefits of the dot product;
    \item \textbf{RotatE}~\cite{RotatE} extends ComplEx by representing entities as complex vectors and relations as rotation operations in a complex vector space;
    
    \item \textbf{R-GCN}\footnote{\url{https://github.com/JinheonBaek/RGCN}}~\cite{R-GCN} uses relation-specific weight matrices that are defined as linear combinations of a set of basis matrices;
    \item \textbf{CompGCN}\footnote{\url{https://github.com/malllabiisc/CompGCN}}~\cite{CompGCN} is a framework for incorporating multi-relational information in graph convolutional networks to jointly embeds both nodes and relations in a graph;
    
    \item \textbf{UKGE}\footnote{\url{https://github.com/stasl0217/UKGE}}~\cite{UKGE} learns embeddings according to the confidence scores of uncertain relation facts to preserve both structural and uncertainty information of facts in the embedding space;
    
    \item \textbf{NSCaching}\footnote{\url{https://github.com/AutoML-Research/NSCaching}}~\cite{NSCaching} is an inexpensive negative sampling approach by using cache to keep track of high-quality negative triplets, which have high scores and rare;
    \item \textbf{SANS}\footnote{\url{https://github.com/kahrabian/SANS}}~\cite{SANS} utilizes the rich graph structure by selecting negative samples from a node’s k-hop neighborhood for negative sampling without  additional parameters and difficult adversarial optimization;
    
    \item \textbf{PUDA}\footnote{\url{https://github.com/lilv98/PUDA}}~\cite{PUDA} is a KGC method to circumvent the impact of the false negative issue by tailoring positive unlabeled risk estimator and address the data sparsity issue by unifying adversarial training and PU learning under the positive-unlabeled minimax game.

\end{itemize}

\subsection{Reproducibility}
\label{ap:implementation}
\subsubsection{Baseline Setup}
All baseline models and \model are trained on the perturbed training set and validated on the perturbed valid set. We utilize \textit{MRR} on the valid set to determine the best models and evaluate them on the clean test set. For uncertain knowledge graph embedding methods UKGE~\cite{UKGE}, since the required uncertainty scores are unavailable, we set the scores for triples in training set as 1 and 0 otherwise. The predicted uncertainty scores produced by UKGE are utilized to rank the potential triples for ranking evaluation. We train all baseline models and \model on the same GPUs (GeForce RTX 3090) and CPUs (AMD Ryzen Threadripper 3970X 32-Core Processor).
 
\subsubsection{\model Setup}
For model training, we utilize Adam optimizer and set the maximum number of epochs as $200$. Within the first $50$ epochs, we disable self-training and focus on learning suboptimal model parameters on noisy and incomplete data. After that, we start the self-training strategy, where the latest label posterior estimation is utilized to sample neighbors for the encoder and select informative unlabeled samples for model training. We set batch size as $256$, the dimensions of all embeddings as $128$, and the dropout rate as $0.5$. For the sake of efficiency, we employ $1$ neighborhood aggregation layer in the encoder. 

For the setting of hyperparameter, we mainly tune positive class prior $\alpha$ in the range of $\{1e-1,5e-2,1e-2,5e-3,1e-3,5e-4,1e-4,5e-5\}$; true positive ratio $\beta$ in the range of $\{0.3,0.2,0.1,0.005,0.001\}$; learning rate in the range of $\{0.02,0.01,0.005,0.001,0.0005\}$; the number of sampled unlabeled triples for each labeled one in the range of $\{50,40,30,20,10\}$. We will publicly release our code and data upon acceptance. 

 \begin{table}[t]
\caption{Overall performance on noisy and incomplete Twitter data, with {\em ptb\_rate = 0.3\/}. Average results on $5$ independent runs are reported. $*$ indicates the statistically significant results over baselines, with $p$-value $<0.01$. The best results are in boldface, and the strongest baseline performance is underlined.}
\label{tb:twitter}
\small
\centering
\resizebox{1.0\linewidth}{!}{
\begin{tabular}{ccccc}
\toprule
\multicolumn{1}{c|}{\textbf{Dataset}} & \multicolumn{4}{c}{\textbf{Twitter}} \\ \hline
\multicolumn{1}{c|}{\textbf{Metrics}} & \textbf{MRR} & \textbf{HIT@100} & \textbf{HIT@50} & \textbf{HIT@30} \\ \hline
\multicolumn{1}{c|}{\textbf{Random}} & 0.001 & 0.006 & 0.003 & 0.002 \\ \hline
\multicolumn{5}{c}{\textit{Knowledge graph embedding methods}} \\ \hline
\multicolumn{1}{c|}{\textbf{TransE}} & 0.010 & 0.091 & 0.058 & 0.041 \\
\multicolumn{1}{c|}{\textbf{TransR}} & 0.009 & 0.078 & 0.048 & 0.033 \\
\multicolumn{1}{c|}{\textbf{DistMult}} & 0.021 & 0.091 & 0.065 & 0.052 \\
\multicolumn{1}{c|}{\textbf{ComplEx}} & {\ul 0.022} & 0.089 & 0.064 & 0.051 \\
\multicolumn{1}{c|}{\textbf{RotatE}} & {\ul 0.022} & {\ul 0.1115} & {\ul 0.077} & {\ul 0.059} \\ \hline
\multicolumn{5}{c}{\textit{Graph neural network methods on KG}} \\ \hline
\multicolumn{1}{c|}{\textbf{RGCN}} & 0.005 & 0.054 & 0.029 & 0.019 \\
\multicolumn{1}{c|}{\textbf{CompGCN}} & 0.014 & 0.089 & 0.059 & 0.044 \\ \hline
\multicolumn{5}{c}{\textit{Uncertain knowledge graph embedding method}} \\ \hline
\multicolumn{1}{c|}{\textbf{UKGE}} & 0.011 & 0.072 & 0.053 & 0.033 \\ \hline
\multicolumn{5}{c}{\textit{Negative sampling methods}} \\ \hline
\multicolumn{1}{c|}{\textbf{NSCaching}} & 0.012 & 0.095 & 0.060 & 0.043 \\
\multicolumn{1}{c|}{\textbf{SANS}} & 0.019 & 0.104 & 0.070 & 0.054 \\ \hline
\multicolumn{5}{c}{\textit{Positive-Unlabeled learning method}} \\ \hline
\multicolumn{1}{c|}{\textbf{PUDA}} & 0.013 & 0.082 & 0.057 & 0.044 \\ \hline
\multicolumn{1}{c|}{\textbf{nPUGraph}} & \textbf{0.030*} & \textbf{0.127*} & \textbf{0.096*} & \textbf{0.074*} \\ \hline
\rowcolor{LightCyan} \multicolumn{1}{c|}{\textbf{Gains ($\%$)}} & \textit{38.2} & \textit{13.9} & \textit{25.3} & \textit{26.5} \\ \bottomrule
\end{tabular}}
\end{table}

\begin{table*}[t]
\caption{Overall performance on noisy and incomplete graphs, with {\em ptb\_rate = 0.1\/}. Average results on $5$ independent runs are reported. $*$ indicates the statistically significant results over baselines, with $p$-value $<0.01$. The best results are in boldface, and the strongest baseline performance is underlined.}
\label{tb:0.1Result}
\small
\centering
\resizebox{0.9\textwidth}{!}{
\fontsize{8.5}{11}\selectfont
\begin{tabular}{c|cccc|cccc|cccc}
\toprule
\textbf{Dataset} & \multicolumn{4}{c|}{\textbf{FB15K}} & \multicolumn{4}{c|}{\textbf{FB15K-237}} & \multicolumn{4}{c}{\textbf{WN18}} \\ \hline
\textbf{Metrics} & \textbf{MRR} & \textbf{H@10} & \textbf{H@3} & \textbf{H@1} & \textbf{MRR} & \textbf{H@10} & \textbf{H@3} & \textbf{H@1} & \textbf{MRR} & \textbf{H@10} & \textbf{H@3} & \textbf{H@1} \\ \hline
\multicolumn{13}{c}{ \textit{Knowledge graph embedding methods}} \\ \hline
\textbf{TransE} & 0.419 & 0.666 & 0.507 & 0.280 & 0.245 & 0.441 & 0.284 & 0.144 & 0.318 & 0.646 & 0.579 & 0.044 \\
\textbf{TransR} & 0.398 & 0.658 & 0.494 & 0.250 & 0.246 & 0.434 & 0.280 & 0.153 & 0.319 & 0.646 & 0.575 & 0.051 \\
\textbf{DistMult} & 0.516 & 0.719 & 0.578 & 0.407 & 0.271 & 0.446 & 0.297 & 0.185 & 0.524 & 0.686 & 0.610 & 0.418 \\
\textbf{ComplEx} & 0.503 & 0.711 & 0.570 & 0.390 & 0.276 & 0.459 & 0.305 & 0.186 & 0.612 & {\ul 0.702} & 0.643 & 0.560 \\
\textbf{RotatE} & {\ul 0.544} & {\ul 0.732} & {\ul 0.608} & {\ul 0.441} & 0.292 & {\ul 0.480} & {\ul 0.324} & 0.199 & 0.613 & 0.691 & 0.642 & 0.567 \\ \hline
\multicolumn{13}{c}{ \textit{Graph neural network methods on KG}} \\\hline
\textbf{RGCN} & 0.196 & 0.372 & 0.209 & 0.110 & 0.169 & 0.317 & 0.177 & 0.097 & 0.483 & 0.625 & 0.554 & 0.396 \\
\textbf{CompGCN} & 0.460 & 0.677 & 0.525 & 0.343 & {\ul 0.293} & 0.475 & {\ul 0.324} & {\ul 0.203} & 0.608 & 0.686 & 0.636 & 0.564 \\ \hline
\multicolumn{13}{c}{ \textit{Uncertain knowledge graph embedding method}} \\\hline
\textbf{UKGE} & 0.338 & 0.425 & 0.321 & 0.233 & 0.231 & 0.411 & 0.204 & 0.110 & 0.381 & 0.541 & 0.407 & 0.331 \\ \hline
\multicolumn{13}{c}{ \textit{Negative sampling method}} \\\hline
\textbf{NSCaching} & 0.495 & 0.689 & 0.557 & 0.390 & 0.153 & 0.305 & 0.167 & 0.080 & 0.434 & 0.542 & 0.470 & 0.374 \\
\textbf{SANS} & 0.422 & 0.649 & 0.493 & 0.298 & 0.271 & 0.453 & 0.301 & 0.182 & {\ul 0.619} & {\ul 0.702} & {\ul 0.644} & {\ul 0.574} \\ \hline
\multicolumn{13}{c}{ \textit{Positive-Unlabeled learning method}} \\\hline
\textbf{PUDA} & 0.493 & 0.713 & 0.559 & 0.377 & 0.271 & 0.443 & 0.298 & 0.185 & 0.520 & 0.667 & 0.608 & 0.419 \\ \hline
\textbf{nPUGraph} & \textbf{0.561*} & \textbf{0.791*} & \textbf{0.621*} & \textbf{0.449*} & \textbf{0.328*} & \textbf{0.535*} & \textbf{0.343*} & \textbf{0.221*} & \textbf{0.630*} & \textbf{0.754*} & \textbf{0.671*} & \textbf{0.599*} \\ \hline
\rowcolor{LightCyan} \textbf{Gains ($\%$)} & \textit{3.0} & \textit{8.1} & \textit{2.1} & \textit{1.9} & \textit{12.0} & \textit{11.4} & \textit{5.9} & \textit{8.7} & \textit{1.8} & \textit{7.4} & \textit{4.1} & \textit{4.4} \\ \bottomrule
\end{tabular}}
\end{table*}

\begin{table*}[t]
\caption{Experimental results under {\em ptb\_rate = 0.5\/}.}
\label{tb:0.5Result}
\small
\centering
\resizebox{0.9\textwidth}{!}{
\fontsize{8.5}{11}\selectfont
\begin{tabular}{c|cccc|cccc|cccc}
\toprule
\textbf{Dataset} & \multicolumn{4}{c|}{\textbf{FB15K}} & \multicolumn{4}{c|}{\textbf{FB15K-237}} & \multicolumn{4}{c}{\textbf{WN18}} \\ \hline
\textbf{Metrics} & \textbf{MRR} & \textbf{H@10} & \textbf{H@3} & \textbf{H@1} & \textbf{MRR} & \textbf{H@10} & \textbf{H@3} & \textbf{H@1} & \textbf{MRR} & \textbf{H@10} & \textbf{H@3} & \textbf{H@1} \\ \hline
\multicolumn{13}{c}{ \textit{Knowledge graph embedding methods}} \\ \hline
\textbf{TransE} & 0.279 & 0.540 & 0.363 & 0.136 & 0.151 & 0.342 & 0.190 & 0.053 & 0.158 & 0.337 & 0.285 & 0.018 \\
\textbf{TransR} & 0.256 & 0.500 & 0.323 & 0.127 & 0.141 & 0.294 & 0.160 & 0.064 & 0.150 & 0.327 & 0.267 & 0.020 \\
\textbf{DistMult} & 0.328 & 0.536 & 0.372 & 0.224 & 0.210 & 0.366 & 0.228 & 0.133 & 0.279 & 0.370 & 0.319 & 0.224 \\
\textbf{ComplEx} & 0.322 & 0.525 & 0.365 & 0.220 & 0.201 & 0.358 & 0.220 & 0.123 & 0.314 & 0.373 & 0.335 & {\ul 0.280} \\
\textbf{RotatE} & 0.350 & 0.547 & 0.398 & 0.249 & {\ul 0.227} & {\ul 0.387} & {\ul 0.246} & {\ul 0.149} & 0.307 & 0.377 & 0.333 & 0.266 \\ \hline
\multicolumn{13}{c}{ \textit{Graph neural network methods on KG}} \\\hline
\textbf{RGCN} & 0.134 & 0.271 & 0.142 & 0.065 & 0.116 & 0.234 & 0.117 & 0.058 & 0.253 & 0.323 & 0.287 & 0.209 \\
\textbf{CompGCN} & {\ul 0.378} & {\ul 0.600} & {\ul 0.429} & {\ul 0.266} & 0.223 & 0.377 & 0.240 & {\ul 0.149} & {\ul 0.345} & 0.315 & 0.336 & 0.279 \\ \hline
\multicolumn{13}{c}{ \textit{Uncertain knowledge graph embedding method}} \\\hline
\textbf{UKGE} & 0.257 & 0.299 & 0.213 & 0.088 & 0.143 & 0.284 & 0.172 & 0.053 & 0.228 & 0.297 & 0.210 & 0.115 \\ \hline
\multicolumn{13}{c}{ \textit{Negative sampling method}} \\\hline
\textbf{NSCaching} & 0.272 & 0.454 & 0.310 & 0.179 & 0.176 & 0.297 & 0.190 & 0.115 & 0.123 & 0.182 & 0.133 & 0.093 \\
\textbf{SANS} & 0.335 & 0.545 & 0.387 & 0.224 & 0.225 & 0.382 & 0.244 & 0.147 & 0.313 & {\ul 0.379} & {\ul 0.337} & 0.275 \\ \hline
\multicolumn{13}{c}{ \textit{Positive-Unlabeled learning method}} \\\hline
\textbf{PUDA} & 0.329 & 0.525 & 0.369 & 0.229 & 0.202 & 0.347 & 0.217 & 0.131 & 0.231 & 0.329 & 0.264 & 0.178 \\ \hline
\textbf{nPUGraph} & \textbf{0.417*} & \textbf{0.663*} & \textbf{0.470*} & \textbf{0.291*} & \textbf{0.258*} & \textbf{0.433*} & \textbf{0.285*} & \textbf{0.171*} & \textbf{0.373*} & \textbf{0.443*} & \textbf{0.379*} & \textbf{0.327*} \\ \hline
\rowcolor{LightCyan} \textbf{Gains ($\%$)} & \textit{10.4} & \textit{10.5} & \textit{9.6} & \textit{9.6} & \textit{13.9} & \textit{12.0} & \textit{16.1} & \textit{14.8} & \textit{8.2} & \textit{16.9} & \textit{12.6} & \textit{16.9} \\ \bottomrule
\end{tabular}}
\end{table*}

\begin{table*}[t]
\caption{Experimental results under {\em ptb\_rate = 0.7\/}.}
\label{tb:0.7Result}
\small
\centering
\resizebox{0.9\textwidth}{!}{
\fontsize{8.5}{11}\selectfont
\begin{tabular}{c|cccc|cccc|cccc}
\toprule
\textbf{Dataset} & \multicolumn{4}{c|}{\textbf{FB15K}} & \multicolumn{4}{c|}{\textbf{FB15K-237}} & \multicolumn{4}{c}{\textbf{WN18}} \\ \hline
\textbf{Metrics} & \textbf{MRR} & \textbf{H@10} & \textbf{H@3} & \textbf{H@1} & \textbf{MRR} & \textbf{H@10} & \textbf{H@3} & \textbf{H@1} & \textbf{MRR} & \textbf{H@10} & \textbf{H@3} & \textbf{H@1} \\ \hline
\multicolumn{13}{c}{ \textit{Knowledge graph embedding methods}} \\ \hline
\textbf{TransE} & 0.219 & 0.457 & 0.294 & 0.087 & 0.104 & 0.273 & 0.128 & 0.020 & 0.114 & 0.229 & 0.199 & 0.020 \\
\textbf{TransR} & 0.195 & 0.396 & 0.248 & 0.087 & 0.096 & 0.217 & 0.108 & 0.036 & 0.096 & 0.207 & 0.167 & 0.016 \\
\textbf{DistMult} & 0.250 & 0.425 & 0.282 & 0.163 & 0.173 & 0.308 & 0.186 & 0.105 & 0.181 & 0.240 & 0.211 & 0.143 \\
\textbf{ComplEx} & 0.241 & 0.408 & 0.271 & 0.157 & 0.158 & 0.290 & 0.170 & 0.092 & 0.202 & 0.243 & 0.219 & 0.178 \\
\textbf{RotatE} & 0.271 & 0.439 & 0.309 & 0.185 & 0.198 & 0.337 & 0.212 & 0.129 & 0.192 & 0.250 & 0.212 & 0.159 \\ \hline
\multicolumn{13}{c}{ \textit{Graph neural network methods on KG}} \\\hline
\textbf{RGCN} & 0.129 & 0.229 & 0.134 & 0.075 & 0.086 & 0.178 & 0.086 & 0.040 & 0.166 & 0.215 & 0.191 & 0.135 \\
\textbf{CompGCN} & {\ul 0.354} & {\ul 0.578} & {\ul 0.402} & {\ul 0.243} & 0.193 & 0.325 & 0.204 & 0.129 & {\ul 0.212} & {\ul 0.262} & {\ul 0.229} & {\ul 0.183} \\ \hline
\multicolumn{13}{c}{ \textit{Uncertain knowledge graph embedding method}} \\\hline
\textbf{UKGE} & 0.186 & 0.358 & 0.199 & 0.076 & 0.133 & 0.201 & 0.115 & 0.075 & 0.099 & 0.176 & 0.153 & 0.116 \\ \hline
\multicolumn{13}{c}{ \textit{Negative sampling method}} \\\hline
\textbf{NSCaching} & 0.174 & 0.309 & 0.194 & 0.105 & 0.157 & 0.276 & 0.169 & 0.099 & 0.057 & 0.085 & 0.062 & 0.041 \\
\textbf{SANS} & 0.292 & 0.478 & 0.332 & 0.196 & {\ul 0.201} & {\ul 0.342} & {\ul 0.216} & {\ul 0.131} & 0.202 & 0.254 & 0.222 & 0.171 \\ \hline
\multicolumn{13}{c}{ \textit{Positive-Unlabeled learning method}} \\\hline
\textbf{PUDA} & 0.254 & 0.421 & 0.283 & 0.170 & 0.165 & 0.291 & 0.176 & 0.104 & 0.109 & 0.171 & 0.127 & 0.077 \\ \hline
\textbf{nPUGraph} & \textbf{0.365*} & \textbf{0.600*} & \textbf{0.427*} & \textbf{0.277*} & \textbf{0.243*} & \textbf{0.390*} & \textbf{0.266*} & \textbf{0.155*} & \textbf{0.247*} & \textbf{0.303*} & \textbf{0.257*} & \textbf{0.209*} \\ \hline
\rowcolor{LightCyan} \textbf{Gains ($\%$)} & \textit{3.0} & \textit{3.8} & \textit{6.3} & \textit{13.9} & \textit{20.9} & \textit{14.1} & \textit{23.0} & \textit{18.5} & \textit{16.8} & \textit{15.5} & \textit{12.3} & \textit{14.4} \\ \bottomrule
\end{tabular}}
\end{table*}
\subsection{Experiments}
\subsubsection{Experimental Results on Twitter Data}
\label{ap:twitter}

We discuss the model performance on noisy and incomplete Twitter data with {\em ptb\_rate = 0.3\/} in this section, which is shown in Table~\ref{tb:twitter}. According to the result of Random, we can infer that all relations on Twitter are $n-n$, where relations can be $1-n$, $n-1$, and $n-n$ for benchmark KG. Therefore, link prediction is more challenging for Twitter data and we adopt Hits at ${30, 50, 100}$ as evaluation metrics, instead. 

For Twitter data, \model achieves impressive performance compared with the baseline models, with $25.98\%$ relative improvement on average. The results of Twitter data support the robustness of \model, which can mitigate false negative/positive issues not only in benchmark KG but also in the real-world social graph.

\subsubsection{Experimental Results under Different Perturbation Rates}
\label{ap:allresult}
The experimental results under perturbation rates 0.1, 0.5, and 0.7 are shown in Table \ref{tb:0.1Result}, Table \ref{tb:0.5Result}, and Table \ref{tb:0.7Result}, respectively. \model outperforms all baseline models for various perturbation rates, demonstrating that \model can mitigate false negative/positive issues on knowledge graphs with different degrees of noise and incompleteness. Notably, comparing these three tables, the relative improvements are more significant under higher {\em ptb\_rate \/} in most cases, showing stronger robustness for \model on graphs with more false negative/positive facts. 




\end{document}